\def\year{2020}\relax
\documentclass[letterpaper]{article} 
\usepackage{aaai20}  
\usepackage{times}  
\usepackage{helvet} 
\usepackage{courier}  
\usepackage[hyphens]{url}  
\usepackage{graphicx} 
\usepackage{graphicx}  
\usepackage[utf8]{inputenc} 
\usepackage[T1]{fontenc}    
\usepackage{url}            
\usepackage{booktabs}       
\usepackage{amsfonts}       
\usepackage{nicefrac}       
\usepackage{microtype}      
\usepackage{mathtools}
\usepackage{amssymb}
\usepackage{amsmath}
\usepackage{enumitem}
\usepackage{subcaption}

\usepackage{booktabs} 
\usepackage{titlesec}
\usepackage[titletoc,toc,title]{appendix}
\usepackage[textsize=small]{todonotes}

\usepackage{amsthm}

\usepackage{algorithmic}
\usepackage{algorithm}

\newcommand{\citet}[1]{\citeauthor{#1} \shortcite{#1}} \newcommand{\citep}{\cite}  

\makeatletter
\DeclareRobustCommand\widecheck[1]{{\mathpalette\@widecheck{#1}}}
\def\@widecheck#1#2{%
    \setbox\z@\hbox{\m@th$#1#2$}%
    \setbox\tw@\hbox{\m@th$#1%
       \widehat{%
          \vrule\@width\z@\@height\ht\z@
          \vrule\@height\z@\@width\wd\z@}$}%
    \dp\tw@-\ht\z@
    \@tempdima\ht\z@ \advance\@tempdima2\ht\tw@ \divide\@tempdima\thr@@
    \setbox\tw@\hbox{%
       \raise\@tempdima\hbox{\scalebox{1}[-1]{\lower\@tempdima\box
\tw@}}}%
    {\ooalign{\box\tw@ \cr \box\z@}}}
\makeatother

\newcommand{\yhat}{\widehat{y}}

\newcommand{\Ptil}{\widetilde{P}}

\newcommand{\Xbf}{\mathbf{X}}

\newcommand{\Phat}{\mathbb{P}}

\newcommand{\Yhat}{\widehat{Y}}

\newcommand{\Pchk}{\mathbb{Q}}

\newcommand{\Dcal}{\mathcal{D}}
\newcommand{\Lcal}{\mathcal{L}}

\newcommand{\Ebb}{\mathbb{E}}
\newcommand{\Ibb}{\mathbb{I}}
\newcommand{\Pbb}{\mathbb{P}}
\newcommand{\Qbb}{\mathbb{Q}}

\newcommand{\xvec}{{\bf x}}

\newcommand{\ebar}{\bar{e}}
\newcommand{\dbar}{\bar{d}}

\ifx\BlackBox\undefined
\newcommand{\BlackBox}{\rule{1.5ex}{1.5ex}}  
\fi

\ifx\QED\undefined
\def\QED{~\rule[-1pt]{5pt}{5pt}\par\medskip}
\fi

\ifx\proof\undefined
\newenvironment{proof}{\par\noindent{\em Proof:\ }}{\hfill\BlackBox\\[.0mm]}
\fi

\ifx\theorem\undefined
\newtheorem{theorem}{Theorem}
\fi

\ifx\example\undefined
\newtheorem{example}{Example}
\fi

\ifx\property\undefined
\newtheorem{property}{Property}
\fi

\ifx\lemma\undefined
\newtheorem{lemma}{Lemma}
\fi

\ifx\proposition\undefined
\newtheorem{proposition}{Proposition}
\fi

\ifx\remark\undefined
\newtheorem{remark}{Remark}
\fi

\ifx\corollary\undefined
\newtheorem{corollary}{Corollary}
\fi

\ifx\definition\undefined
\newtheorem{definition}{Definition}
\fi

\ifx\conjecture\undefined
\newtheorem{conjecture}{Conjecture}
\fi

\ifx\axiom\undefined
\newtheorem{axiom}[theorem]{Axiom}
\fi

\ifx\claim\undefined
\newtheorem{claim}[theorem]{Claim}
\fi

\ifx\assumption\undefined
\newtheorem{assumption}{Assumption}
\fi

\newcommand*{\argmin}{\mathop{\mathrm{argmin}}}

\title{Fairness for Robust Log Loss Classification}
\author{Ashkan Rezaei\textsuperscript{\rm 1}\thanks{These two authors contributed equally.},
Rizal Fathony\textsuperscript{\rm 2}\footnotemark[1],
Omid Memarrast\textsuperscript{\rm 1},
Brian Ziebart\textsuperscript{\rm 1}\\ \textsuperscript{\rm 1} Department of Computer Science, University of Illinois at Chicago \\
\textsuperscript{\rm 2} School of Computer Science, Carnegie Mellon University
\\ arezae4@uic.edu, rfathony@cs.cmu.edu,
omemar2@uic.edu, bziebart@uic.edu }
\begin{document}

\maketitle

\begin{abstract}
Developing classification methods with high accuracy that also avoid unfair treatment of different groups has become increasingly important for data-driven decision making in social applications. 
Many existing methods enforce fairness constraints on a selected classifier (e.g., logistic regression) by directly forming constrained optimizations.
We instead re-derive a new classifier from the first principles of distributional robustness that incorporates fairness criteria into 
a worst-case logarithmic loss minimization.
This construction takes the form of a minimax game and produces a parametric exponential family conditional distribution that resembles truncated logistic regression.
We present the theoretical benefits of our approach in terms of its convexity and asymptotic convergence. 
We then demonstrate the practical advantages of our approach on three benchmark fairness datasets. 

\end{abstract} 

\section{Introduction}
Though maximizing accuracy has been the principal objective for classification tasks, competing priorities are also often of key concern in practice.  
Fairness properties that guarantee equivalent treatment to different groups 
in various ways are a prime example. 
These may be desirable---or even legally required---when making admissions decisions for universities 
\cite{chang2006applying,kabakchieva2013predicting},
employment and promotion decisions for organizations \cite{lohr2013big}, medical decisions for hospitals and insurers \cite{shipp2002diffuse,obermeyer2016predicting}, sentencing guidelines within the judicial system \cite{moses2014using,o2016weapons}, loan decisions for the financial industry 
\cite{shaw1988using,carter1987assessing}
and in many other applications.  
\emph{Group fairness} criteria generally partition the population based on a protected attribute into groups and mandate equal treatment of members across groups based on some defined statistical measures. 
We focus on 
three prevalent group fairness measures in this paper:   
\emph{demographic parity} \cite{calders2009building}, \emph{equalized odds}, 
and \emph{equalized opportunity} \cite{hardt2016equality}.

Techniques for constructing predictors with group fairness properties can be categorized into \emph{pre-}, \emph{post-}, and \emph{in-processing} methods.
\emph{Pre-processing methods} use reweighting and relabeling \cite{Kamiran2012,krasanakis2018adaptive} or other transformations of input data \cite{Calmon2017,Zemel13,feldman2015certifying,del2018obtaining,donini2018empirical,zhang2018achieving} 
to remove unfair dependencies 
with protected attributes. 
\emph{Post-processing methods} adjust the class labels (or label distributions) provided from black box classifiers to satisfy desired fairness criteria
\cite{hardt2016equality,Pleiss17,hacker2017continuous}. 
\emph{In-processing methods} integrate fairness criteria into the optimization procedure of the classifier with constraints/penalties \cite{donini2018empirical,zafar2017aistats,zafar2017fairness,zafar2017parity,cotter2018,goel2018non,woodworth2017learning,kamishima2011fairness,bechavod2017penalizing,quadrianto2017recycling}, 
meta-algorithms \cite{celis2019classification,menon2018cost}, reduction-based methods \cite{agarwal2018reductions}, or generative-adversarial training \cite{madras2018learning,zhang2018mitigating,celis2019improved,xu2018fairgan,adel2019one}.


Unlike many existing methods that directly form a constrained optimization from base classifiers, we take a step back and re-derive prediction from the underlying formulation of logistic regression.
Working from the first principles of distributionally robust estimation \cite{topsoe1979information,grunwald2004game,delage2010distributionally}, we 
incorporate fairness constraints into the formulation of the predictor.
We 
pose predictor selection as a minimax game between a predictor that is fair on a training sample and a worst-case approximator of the training data labels that maintains some statistical properties of the training sample.
Like \emph{post-processing methods}, our approach reshapes its predictions for each group to satisfy fairness requirements.
However, our approach is inherently an \emph{in-process method} that jointly optimizes this fairness transformation and linear feature-based parameters for an exponential family distribution that can be viewed as 
truncated logistic regression. 
Our method assumes group membership attributes are given at training and testing time, which matches many real-world applications. We leave the extension of our approach to settings with inferred group attributes as future work.

Our method reduces to a convex optimization problem
with a unique solution for resolving unfairness between groups that asymptotically minimizes the KL divergence from the true distribution.
In contrast,
many existing methods are susceptible to the local optima of non-convex optimization or to the approximation error from relaxations
\cite{zafar2017fairness,zafar2017aistats,cotter2018,woodworth2017learning,kamishima2011fairness,bechavod2017penalizing,quadrianto2017recycling}, 
do not have unique solutions for ensuring fairness \cite{hardt2016equality}, or produce mixtures of predictors  \cite{agarwal2018reductions} rather than a single coherent predictor.
For fairness criteria that include the true label (e.g., equalized opportunity, equalized odds),
we introduce a method for making predictions from label-conditioned distributions and establish desirable asymptotic properties. 
We demonstrate the practical advantages of our  approach compared to existing fair classification methods on benchmark data-driven decision tasks.

\section{Background}
\subsection{Measures of fairness for decision making}
Several useful measures have been proposed to quantitatively assess fairness in decision making.
Though our approach can be applied to a wider range of fairness constraints,
we focus on three prominent ones:
 \emph{Demographic Parity} \cite{calders2009building}, \emph{Equality of Opportunity} \cite{hardt2016equality} and \emph{Equality of Odds} \cite{hardt2016equality}. 
These are defined for binary decision settings with examples drawn from 
a population distribution, $({\bf X}, A, Y) \sim P$, with
$\widetilde{P}({\bf x}, a, y)$ denoting this empirical sample distribution, $\{{\bf x}_i, a_i, y_i\}_{i=1:n}$.
Here, $y=1$ is the ``advantaged'' class for positive decisions. 
Each example also possesses a protected attribute $a \in \{0,1\}$ that defines membership in one of two groups.
The general decision task is to construct a probabilistic prediction, $\Pbb(\yhat|\xvec,a)$ over the
decision variable $\yhat \in \{0,1\}$ given 
input ${\bf x} \in \boldsymbol{\mathcal{X}}$ and training data $\widetilde{P}({\bf x}, a, y)$.
We similarly notate an adversarial conditional distribution that approximates the labels as $\Qbb$.
$\Pbb$ and $\Qbb$ are the key objects being optimized in our formulation.


Fairness requires treating the different groups equivalently in various ways. 
Unfortunately, the na\"ive approach of excluding the protected attribute from the decision function, e.g., restricting to $\Pbb(\yhat|{\xvec})$, does not guarantee fairness because the protected attribute $a$ may still be inferred from ${\bf x}$ \cite{dwork2012fairness}.
Instead of imposing constraints on the predictor's inputs, definitions of fairness require statistical properties on its decisions to hold.

\begin{definition} \label{def:demographic}
A classifier satisfies \textsc{Demographic Parity (D.P.)} if the output variable $\widehat{Y}$ is statistically independent of the protected attribute $A$:
$P(\widehat{Y}\! =\! 1| A\! =\! a) = P(\widehat{Y}\! =\! 1),
\hspace{0.5em}
\forall 
a \in \{0,1\}$.
\end{definition}

\begin{definition} \label{def:odds}
A classifier satisfies \textsc{Equalized Odds (E.Odd.)} if the output variable $\widehat{Y}$ is conditionally independent of the protected attribute $A$ given the true label $Y$: 
$P(\widehat{Y}\! =\! 1 | A\! =\! a , Y\! =\! y) =  P(\widehat{Y}\! =\! 1 | Y\! =\! y), 
\hspace{0.5em} \forall 
y,a \in \{0,1\}.$ 
\end{definition}
\begin{definition} \label{def:opportunity}
A classifier satisfies \textsc{Equalized Opportunity (E.Opp.)} 
if the output variable $\widehat{Y}$ and protected attribute $A$ are conditionally independent given $Y = 1$: 
$P(\widehat{Y}\! =\! 1 | A\! =\! a , Y\! =\! 1) = \; 
P(\widehat{Y}\! =\! 1 | Y\! =\! 1), 
\hspace{0.5em} \forall 
a \in \{0,1\}$. 
\end{definition}

The sets of decision functions $\Pbb$ satisfying these fairness constraints are convex and can be defined using linear constraints \cite{agarwal2018reductions}. The general form for these constraints is:
\begin{align}
\Gamma: \Big\{ \Pbb \mid &
\tfrac{1}{p_{\gamma_1}} \Ebb_{ \substack{ \Ptil({\bf  x},a,y) \\
            \Pbb(\yhat|{\bf x},a,y)}
    }[\Ibb(\Yhat\!=\!1\wedge \gamma_1(A,Y))] \notag \\     
    &= \tfrac{1}{p_{\gamma_0}} \Ebb_{ \substack{ \Ptil({\bf x},a,y) \\
            \Pbb(\yhat|{\bf x},a,y)}
    }[\Ibb(\Yhat\!=\!1 \wedge \gamma_0(A,Y))] \label{eq:fairconstraints}
\Big\}, 
\end{align}
where $\gamma_1$ and $\gamma_0$ denote some combination of group membership and ground-truth class for each example, while $p_{\gamma_1}$ and $p_{\gamma_0}$ denote the empirical frequencies of $\gamma_1$ and $\gamma_0$: 
$p_{\gamma_i} = \Ebb_{\Ptil(a,y)}[\gamma_i(A,Y)]$. 
We specify $\gamma_1$ and $\gamma_0$ in \eqref{eq:fairconstraints} for fairness constraints (Definitions 1, 2, and 3) as:
\begin{align}
 &   \Gamma_{\text{dp}} \iff \gamma_j(A,Y)=\Ibb(A=j); \\  
 &   \Gamma_{\text{e.opp}} \iff \gamma_j(A,Y)=\Ibb(A=j \wedge Y=1); \\
 &   \Gamma_{\text{e.odd}} \iff \gamma_j(A,Y)=\left[\begin{array}{c}\Ibb(A=j \wedge Y=1)\\\Ibb(A=j \wedge Y=0)\end{array}\right].
\end{align}

\subsection{Robust log-loss minimization, maximum entropy, and logistic regression}
\label{sec:logloss}

The logarithmic loss, $-\sum_{{\bf x},y} P({\bf x},y) \log \Pbb(y|{\bf x})$, is an information-theoretic measure of the expected amount of ``surprise'' (in bits for $\log_2$) that the predictor, $\Pbb(y|{\bf x})$, experiences when encountering labels $y$ distributed according to $P({\bf x},y)$.
Robust minimization of the logarithmic loss serves a fundamental role in constructing exponential probability distributions (e.g., Gaussian, Laplacian, Beta, Gamma, Bernoulli \cite{lisman1972note}) and predictors \cite{manning2003optimization}.
For conditional probabilities, it is 
equivalent to maximizing the conditional entropy
\cite{jaynes1957information}:
\begin{align}
    & \min_{\Pbb(\yhat|{\bf x}) \in \Delta} \max_{\Qbb(\yhat|{\bf x}) \in \Delta \cap \Xi}
    \!\!-\sum_{{\bf x},\yhat} \Ptil({\bf x}) \Qbb(\yhat|{\bf x}) \log \Pbb(\yhat|{\bf x})\label{eq:maxent} \\ \notag
    & 
    = \!\!\! \max_{\Pbb(\yhat|{\bf x}) \in \Xi} 
    \!\!
    -\!\sum_{{\bf x},\yhat} \Ptil({\bf x}) \Pbb(\yhat|{\bf x}) \log \Pbb(\yhat|{\bf x})
    = \!\!
    \max_{\Phat(\yhat|{\bf x}) \in \Xi}
    H(\Yhat|{\bf X})
, 
\end{align}
after simplifications based on the fact that the saddle point solution is $\Pbb = \Qbb$.
When the loss maximizer $\Qbb$ is constrained to match the statistics of training data (specified using vector-valued feature function $\phi$), 
{\small
\begin{align}
&\Xi: \Big\{ \Qbb \mid 
 \Ebb_{
        \Ptil({\bf x}); 
            \Qbb(\yhat|{\bf x})
    }[\phi
    ({\bf X},\Yhat)] = 
     \Ebb_{\Ptil({\bf x},y)}\left[\phi({\bf X},Y) \right]
 \Big\}, \label{eq:featurematch}
\end{align}}%
the robust log loss minimizer/maximum entropy predictor (Eq. \eqref{eq:maxent}) is the logistic regression model,
$P(y|{\bf x}) \propto e^{\theta^\mathrm{T}\phi({\bf x},y)}$, with $\theta$ estimated by maximizing data likelihood \cite{manning2003optimization}.
While this distribution technically needs to only be defined at input values in which training data exists (i.e., $\Ptil({\bf x}) > 0$), we employ an inductive assumption that generalizes the form of the distribution to other inputs.

This formulation has been leveraged to provide robust predictions under covariate shift (i.e., difference in training and testing distributions) 
\cite{liu2014robust}
and for constructing consistent predictors for multiclass classifications \cite{fathony2018consistent} and graphical models \cite{fathony2018distributionally}.
Our approach similarly extends this fundamental formulation by imposing fairness constraints on $\Pbb$. 
However, 
since the fairness constraints and statistic-matching constraints are often not fully compatible (i.e., $\Gamma \not\subseteq \Xi$), the saddle point solution is no longer simple (i.e., $\Pbb\neq \Qbb$).

\section{Formulation and Algorithms}

Given fairness requirements for a predictor (Eq. \eqref{eq:fairconstraints}) and partial knowledge of the population distribution provided by a training sample (Eq. \eqref{eq:featurematch}), how should a fair predictor be constructed?
Like all inductive reasoning, good performance on a known training sample does not ensure good performance on the unknown population distribution.
We take a robust estimation perspective 
by seeking the best solution for the worst-case population distribution under these constraints.

\subsection{Robust and fair log loss minimization}

We 
formulate
the robust fair predictor's construction as a minimax game between the predictor and a worst-case approximator of the population distribution.
We assume the availability of a set of training samples,
$\{({\bf x}_i,a_i,y_i)\}_{i=1:n}$, which we equivalently denote by probability distribution $\Ptil({\bf x},a,y)$.
\begin{definition} \label{def:fairRobust}
The {\bf Fair Robust Log-Loss Predictor}, $\Pbb$, minimizes the worst-case log loss---as chosen by approximator $\Qbb$ constrained to reflect training statistics (denoted by set $\Xi$ of Eq. \eqref{eq:featurematch})---while providing empirical fairness guarantees\footnote{$\Delta$ is the set of conditional probability simplexes  
(i.e., $\Pbb(y|{\bf x},a) \geq 0, 
\sum_{y'} \Pbb(y'|{\bf x},a) = 1, \forall {\bf x}, y, a$).} 
(denoted by set $\Gamma$ of Eq. \eqref{eq:fairconstraints}):
\begin{align}
    \min_{\Pbb \in  \Delta \cap \Gamma} \max_{\Qbb \in \Delta \cap \Xi} & \Ebb_{\substack{
        \Ptil({\bf x},a,y)\\
            \Qbb(\yhat|{\bf x},a,y)}} 
            \left[-\log \Pbb(\Yhat|{\bf X},A,Y)\right].
    \label{eq:definition}
\end{align}
\end{definition}

Though conditioning the decision variable $\Yhat$ on the true label $Y$ would appear to introduce a trivial solution ($\Yhat=Y$), instead, $Y$  only influences $\Yhat$ 
based on fairness properties due to the robust predictor's construction. 
Note that if the fairness constraints do not relate $Y$ and $\Yhat$, the resulting distribution is
conditionally independent 
(i.e., $\Pbb(\Yhat|{\bf X},A,Y=0)=\Pbb(\Yhat|{\bf X},A,Y=1)$), and when all fairness constraints are removed, this formulation reduces to the familiar logistic regression model \cite{manning2003optimization}.
Conveniently,
this saddle point problem 
is convex-concave in $\Pbb$ and $\Qbb$ with additional convex constraints ($\Gamma$ and $\Xi$) on each distribution.

\subsection{Parametric Distribution Form}

By leveraging strong minimax duality in the ``log-loss game'' \cite{topsoe1979information,grunwald2004game} and strong Lagrangian duality \cite{boyd2004convex}, we derive the parametric form of our predictor.
\!\!\footnote{The proofs of Theorem \ref{thm:parametric} and other theorems in the paper are available in the supplementary material.} 
\begin{theorem}
\label{thm:parametric}
The {\bf Fair Robust Log-Loss Predictor} (Definition \ref{def:fairRobust}) has equivalent 
dual formulation:
\begin{align}
 \min_{\theta} & \max_{\lambda} \; \tfrac{1}{n}\!\!\!\!\!\!
    \sum_{({\bf x},a,y) \in \Dcal}\!\! \bigg\{
    \Ebb_{
        \Qbb_{\theta,\lambda} (\yhat|{\bf x},a,y) } 
            \!\left[-\log \Pbb_{\theta,\lambda}(\Yhat|{\bf x},a,y)\right] \notag \\ 
            &  
            + \theta^\top\!\! \left( \Ebb_{ \Qbb_{\theta,\lambda}(\yhat|{\bf x},a,y)}
    [\phi({\bf x},\Yhat)]
    - \phi({\bf x},y) \right) \notag
    \\
& + \lambda \big( \tfrac{1}{p_{\gamma_1}} \Ebb_{ \Pbb_{\theta,\lambda}(\yhat|{\bf x},a,y)}
    [\Ibb(\Yhat\!=\!1\wedge \gamma_1(A,Y))]  \notag \\ &  
    - \tfrac{1}{p_{\gamma_0}} \Ebb_{ \Pbb_{\theta,\lambda}(\yhat|{\bf x},a,y)}
    [\Ibb(\Yhat\!=\!1 \wedge \gamma_0(A,Y))]  \big) \bigg\}, \label{eq:dual} 
\end{align}%
with Lagrange multipliers $\theta$ and  $\lambda$ for moment matching and fairness constraints, respectively, and $n$ 
samples in the dataset. The parametric distribution of $\Pbb$ is: 
{\small
\begin{align}
    & \Pbb_{\theta,\lambda}(\yhat = 1|{\bf x},a,y) = \label{eq:truncate}
 \\ & 
    \begin{cases}
        \min\big\{
        e^{\theta^\top \phi({\bf x},1)}/ Z_{\theta}({\bf x}), \frac{p_{\gamma_1}}{\lambda} \big\} & \text{if } 
        \gamma_1(a,y) \land \lambda > 0\\
        \max \big\{
        e^{\theta^\top \phi({\bf x},1)}/ Z_{\theta}({\bf x}), 1\!-\!\frac{p_{\gamma_0}}{\lambda} \big\} & \text{if } 
        \gamma_0(a,y) \land \lambda > 0\\
        \max\big\{
        e^{\theta^\top \phi({\bf x},1)}/ Z_{\theta}({\bf x}), 1 \!+\! \frac{p_{\gamma_1}}{\lambda} \big\} & \text{if } 
        \gamma_1(a,y) \land \lambda <0\\
        \min \big\{
        e^{\theta^\top \phi({\bf x},1)}/ Z_{\theta}({\bf x}),
        -\frac{p_{\gamma_0}}{\lambda} \big\} & \text{if } 
        \gamma_0(a,y) \land \lambda <0\\
        e^{\theta^\top \phi({\bf x},1)}/ Z_{\theta}({\bf x})& \text{otherwise}, 
    \end{cases} \notag
\end{align}}%
where $Z_{\theta}({\bf x})= e^{\theta^\top \phi({\bf x},1)} + e^{\theta^\top \phi({\bf x},0)}$ is the normalization constant. 
The parametric distribution of $\Qbb$ is defined using the following relationship with $\Pbb$:
{
\begin{align}
    & \Qbb_{\theta,\lambda}(\yhat=1|{\bf x},a,y) =  \Pbb_{\theta,\lambda}(\yhat=1|{\bf x},a,y)
    \times
    \label{eq:pcheck} \\   & \qquad 
    \begin{cases}
         \big( 1+\frac{\lambda}{p_{\gamma_1}} \Pbb_{\theta,\lambda}(\yhat=0|{\bf x},a,y) \big) &  \text{if } \gamma_1(a,y)\\
        \big( 1-\frac{\lambda}{p_{\gamma_0}} \Pbb_{\theta,\lambda}(\yhat=0|{\bf x},a,y) \big) & \text{if } \gamma_0(a,y) \\
         1 & \text{otherwise}. \notag  
    \end{cases}
\end{align}
}
\end{theorem}

Note that the predictor's 
distribution is a member of the exponential family that is similar to 
standard binary logistic regression, but with the option to \emph{truncate} the probability based on the value of $\lambda$. The truncation of $\Pbb_{\theta,\lambda}(\yhat = 1|{\bf x},a,y)$ 
is from above 
when $0 \!<\! {p_{\gamma_1}}/{\lambda} \!<\! 1$ and $\gamma_1(a,y)\!=\!1$, 
and from below 
when $-1 \!<\! {p_{\gamma_1}}/{\lambda} \!<\! 0$ and $\gamma_1(a,y)\!=\!1$.
The approximator's 
distribution is computed from the predictor's distribution using the quadratic function in Eq. \eqref{eq:pcheck}, e.g., in the case where $\gamma_1(a,y)\!=\!1$: 
{\small
\[
    \Qbb_{\theta,\lambda}(\yhat=1|{\bf x},a,y) = 
    \rho (1 + \tfrac{\lambda}{p_{\gamma_1}}(1-\rho) )
        = (1 +\tfrac{\lambda}{p_{\gamma_1}}) \rho - \tfrac{\lambda}{p_{\gamma_1}} \rho^2, 
        \notag
\]}%
where $\rho \triangleq \Pbb_{\theta,\lambda}(\yhat\!=\!1|{\bf x},a,y)$.
\begin{figure}[tb]
\rotatebox{90}{\small $\qquad\qquad\qquad\quad \Qbb_\theta(\Yhat=1|{\bf x},a,y)$}
\centering
\includegraphics[width=0.93\columnwidth]{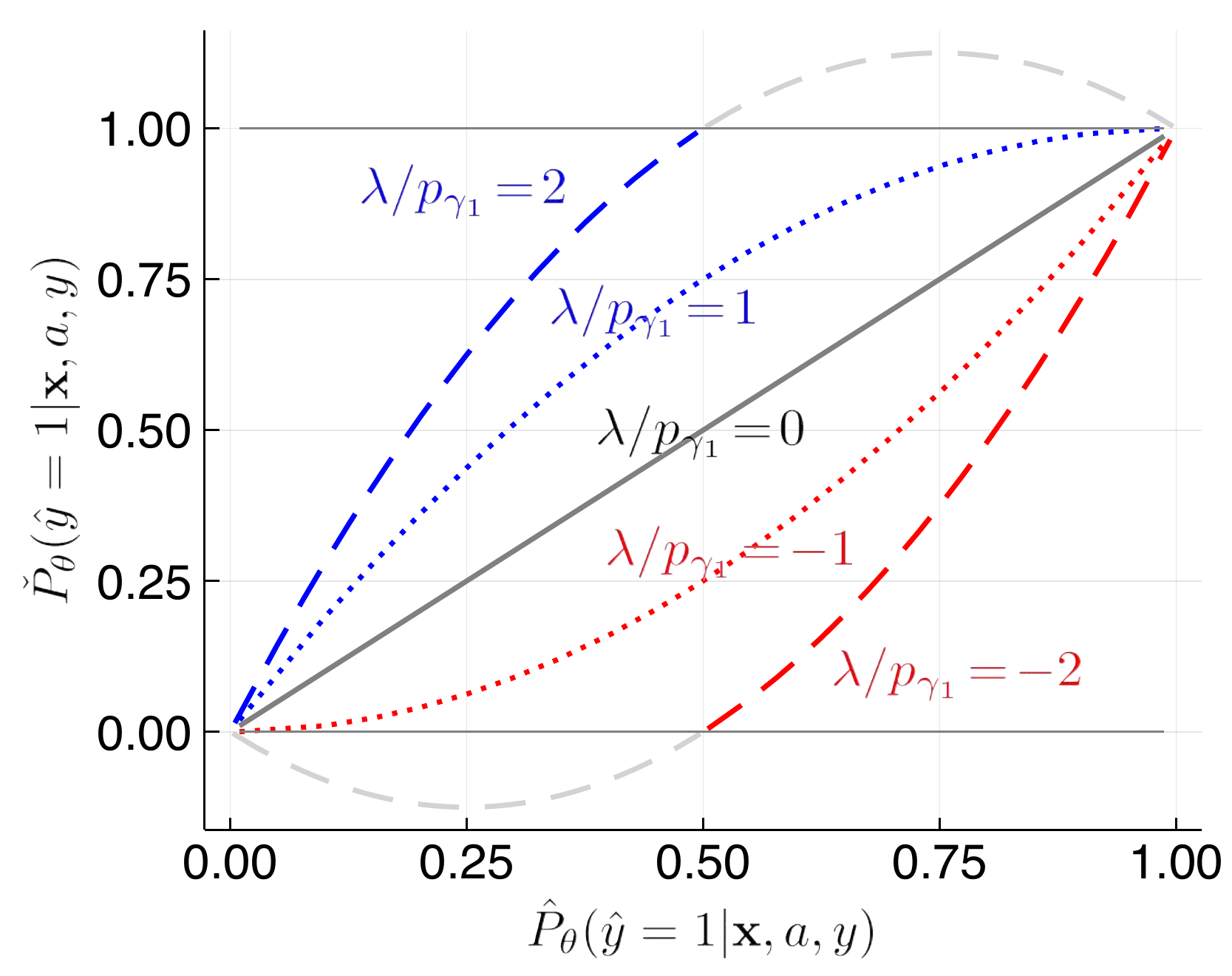}
\small \hspace*{8mm}
$\Pbb_\theta(\Yhat=1|{\bf x},a,y)$
\caption{The relationship between predictor and approximator's 
distributions, $\Pbb$ and $\Qbb$.}
\label{fig:pcheckhat}
\end{figure}
Figure \ref{fig:pcheckhat} illustrates the relationship between $\Pbb_{\theta,\lambda}(\yhat\!=\!1|{\bf x},a,y)$ and $\Qbb_{\theta,\lambda}(\yhat\!=\!1|{\bf x},a,y)$ for decisions influencing the fairness of group one (i.e., $\gamma_1(a,y)=1$). When ${\lambda}/{p_{\gamma_1}}\!=\!0$, the approximator's probability is equal to the predictor's probability as shown in the plot as a straight line.
Positive values of $\lambda$ curve the function upward 
(
e.g., ${\lambda}/{p_{\gamma_1}}\!=\!1$) as shown in the plot. 
For larger $\lambda$ (e.g., ${\lambda}/{p_{\gamma_1}}\!=\!2$), some of the valid predictor probabilities ($0 < \Pbb < 1$) map to invalid approximator probabilities (i.e., $\Qbb \geq 1$) according to the quadratic function. 
In this case (e.g., ${\lambda}/{p_{\gamma_1}}\!=\!2$ and $\Pbb_{\theta,\lambda}(\yhat\!=\!1|{\bf x},a,y) > 0.5$), the predictor's probability is truncated to ${p_{\gamma_1}}/{\lambda}\!=\!0.5$
according to Eq. \eqref{eq:truncate}.
Similarly, for negative $\lambda$, the curve is shifted downward and the predictor's probability is truncated when the quadratic function mapping results in a negative value of $\Qbb$. 
When $\gamma_0(a,y)=1$, the reverse shifting is observed, i.e., shifting downward when $\lambda > 0$ and shifting upward when $\lambda < 0$. 

We contrast our reshaping function of the decision distribution (Figure \ref{fig:pcheckhat}) with the \emph{post-processing method} of \citet{hardt2016equality} shown in Figure \ref{fig:postproc}.
Here, we use $\Qbb(\Yhat=1|{\bf x},a)$ to represent the estimating distributions
(the approximator's distribution in our method, and the standard logistic regression in \citet{hardt2016equality})
and the post-processed predictions as $\Pbb(\Yhat=1|{\bf x},a)$.
Both 
shift the positive prediction rates of each group to provide fairness.
However, our approach provides a monotonic and parametric transformation, avoiding the criticisms that \citet{hardt2016equality}'s modification (flipping some decisions) is partially random, creating an unrestricted hypothesis class \cite{bechavod2017penalizing}. 
Additionally, since our parametric reshaping function is learned within an \emph{in-processing} method, it avoids the noted suboptimalities that have been established for certain population distributions when employing post-processing alone \cite{woodworth2017learning}. 

\begin{figure}[tb]
\centering
\rotatebox{90}{\small $\qquad\qquad\qquad\quad \Qbb(\Yhat=1|{\bf x},a)$}
\includegraphics[width=0.9\columnwidth]{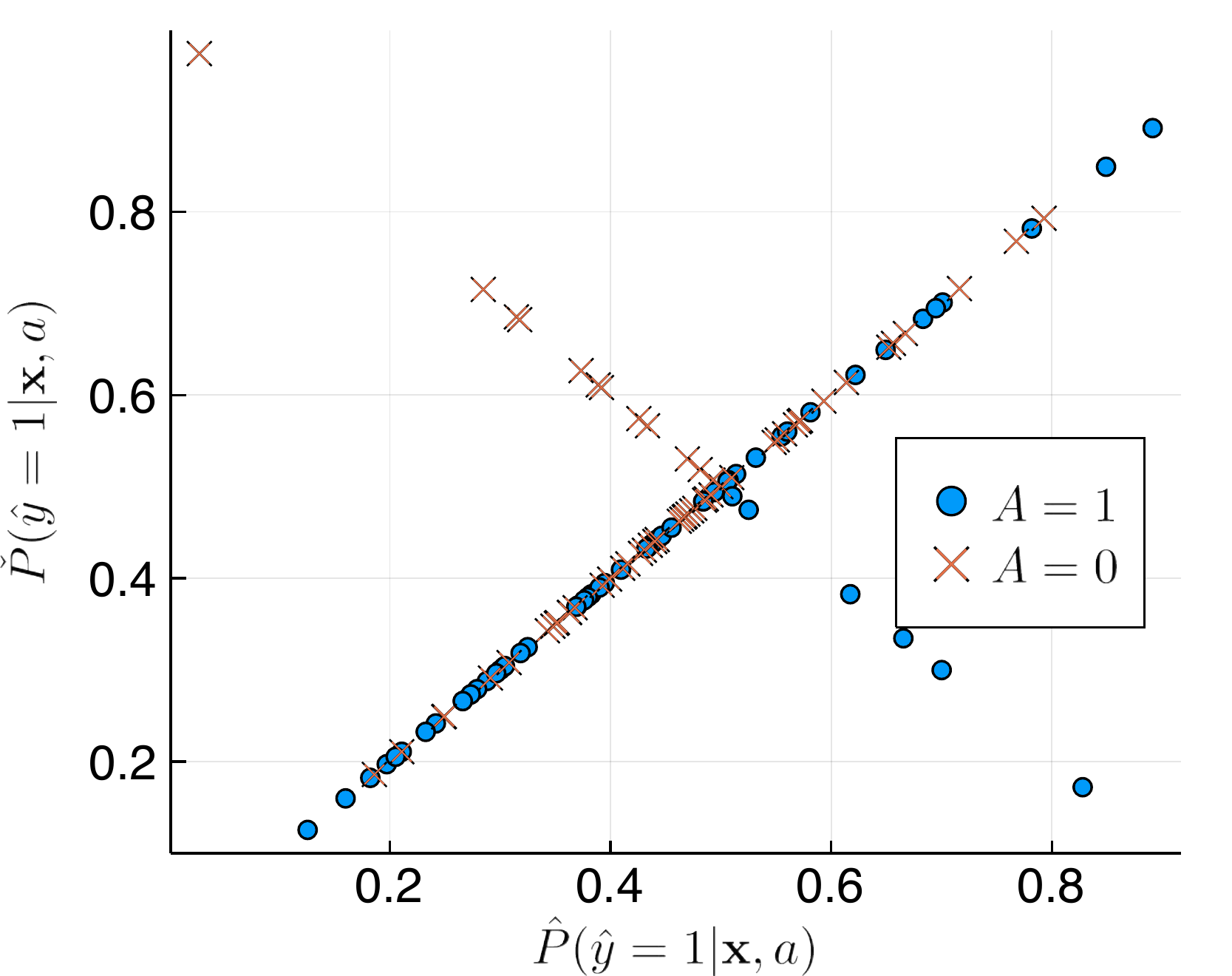}
\small \hspace*{8mm}
$\Pbb(\Yhat=1|{\bf x},a)$
\caption[
Post-processing correction of logistic regression 
\cite{Pleiss17,hardt2016equality} on the COMPAS dataset.]{Post-processing correction\footnotemark{ } 
of logistic regression 
\cite{Pleiss17,hardt2016equality} on the COMPAS dataset.
}
\label{fig:postproc}
\end{figure}
\footnotetext{{\url{https://github.com/gpleiss/equalized_odds_and_calibration}}}

\subsection{Enforcing fairness constraints}

The inner maximization in Eq. \eqref{eq:dual} 
finds the optimal $\lambda$ that enforces the fairness constraint. From the perspective of the parametric distribution of $\Pbb$, this is equivalent to finding threshold points (e.g., ${p_{\gamma_1}}/{\lambda}$ and $1 - {p_{\gamma_0}}/{\lambda}$) in the $\min$ and $\max$ function of Eq. \eqref{eq:truncate} such that the expectation of the truncated exponential probabilities of $\Pbb$ in group $\gamma_1$ match the one in group $\gamma_0$. Given the value of $\theta$, we find the optimum $\lambda^*$ directly by finding the threshold points. 
We first compute the exponential probabilities $P_{e}(\yhat=1|{\bf x},a,y) = \exp(\theta^\top \phi({\bf x},1)) / {Z_{\theta}({\bf x})}$ for each examples in $\gamma_1$ and $\gamma_0$. Let $E_1$ and $E_0$ be the sets that contain $P_e$ for group $\gamma_1$ and $\gamma_0$ respectively.
Finding $\lambda^*$ given the sets $E_1$ and $E_0$ requires sorting the probabilities for each set, and then iteratively finding the threshold points for both sets 
simultaneously. We refer 
to the supplementary material 
for the detailed algorithm.

\subsection{Learning}

Our learning process seeks 
parameters $\theta,\lambda$ for our distributions ($\Pbb_{\theta,\lambda}$ and $\Qbb_{\theta,\lambda}$) that 
match the statistics of the approximator's distribution with training data ($\theta$) and provide fairness ($\lambda$), as illustrated in Eq. \eqref{eq:dual}. 
Using our algorithm from 
the previous subsection 
to directly compute the best $\lambda$ given arbitrary values of $\theta$, denoted $\lambda^*_\theta$, 
the optimization of Eq. \eqref{eq:dual} reduces to a simpler optimization solely over $\theta$, as described in Theorem \ref{thm:objective}.

\begin{theorem}
\label{thm:objective}
Given the optimum value of $\lambda^*_\theta$ for $\theta$, the dual formulation in Eq. \eqref{eq:dual} reduces to: 
{\small
\begin{align}
 &\min_{\theta}  \tfrac{1}{n}
    \textstyle\sum_{({\bf x},a,y) \in \Dcal} \ell_{\theta,\lambda^*_\theta}({\bf x},a,y), \quad \text{ where: }
    \label{eq:objective} \\
    &\ell_{\theta,\lambda^*}({\bf x},a,y) = 
    - \theta^\top \phi({\bf x},y) + \notag \\ &
    \begin{cases}
        -\log(\frac{p_{\gamma_1}}{\lambda^*_\theta}) + \theta^\top ( \phi({\bf x},1) 
        ) & \text{if } \gamma_1(a,y) \wedge T({\bf x},\theta) \wedge \lambda^*_\theta > 0 \\
        -\log(\frac{p_{\gamma_0}}{\lambda^*_\theta}) + \theta^\top ( \phi({\bf x},0) 
        ) & \text{if } \gamma_0(a,y) \wedge T({\bf x},\theta) \wedge \lambda^*_\theta > 0 \\
        -\log(-\frac{p_{\gamma_1}}{\lambda^*_\theta}) + \theta^\top ( \phi({\bf x},0) 
        ) & \text{if } \gamma_1(a,y) \wedge T({\bf x},\theta) \wedge \lambda^*_\theta < 0\\
        -\log(-\frac{p_{\gamma_0}}{\lambda^*_\theta}) + \theta^\top ( \phi({\bf x},1) 
        ) & \text{if } \gamma_0(a,y) \wedge T({\bf x},\theta) \wedge \lambda^*_\theta < 0 \\
        \log Z_\theta({\bf x})  & \text{otherwise}. \notag
    \end{cases}
\end{align}}%
Here, 
$T({\bf x},\theta)\triangleq 1$ if the exponential probability is truncated (for example when $e^{\theta^\top \phi({\bf x},1)}/{Z_{\theta}({\bf x})} > {p_{\gamma_1}}/{\lambda^*_\theta}$, $\gamma_1(a,y)=1$, and $\lambda^*_\theta > 0$), and is $0$ otherwise.
\end{theorem}

We present an important optimization property for our objective function in the following theorem.
\begin{theorem}
\label{thm:convex}
The objective function in Theorem \ref{thm:objective} (Eq. \eqref{eq:objective}) is convex with respect to $\theta$.
\end{theorem}

To improve the generalizability of our parametric model, we employ a standard L2 regularization technique that is common for 
logistic regression models: 
$\theta^* = \textstyle\argmin_{\theta} 
    \textstyle\sum_{({\bf x},a,y) \in \Dcal} \ell_{\theta,\lambda^*_\theta}({\bf x},a,y) + \tfrac{C}{2} \| \theta \|_2^2,
$ 
where $C$ is the regularization constant. We employ a standard batch gradient descent optimization algorithm (e.g., L-BFGS) to obtain a solution for $\theta^*$.\!\footnote{We refer the reader to the supplementary material 
for 
details.} 
We also compute the corresponding solution for the inner optimization, $\lambda^*_{\theta^*}$, and then construct the optimal predictor and approximator's parametric distributions based on the values of $\theta^*$ and $\lambda^*_{\theta^*}$.

\subsection{Inference}
\label{sec:inference}

In the inference step, we apply the optimal parametric predictor distribution $\Pbb_{\theta^*,\lambda^*_{\theta^*}}$ 
to new example inputs $({\bf x},a)$ in the testing set. Given the value of $\theta^*$ and $\lambda^*_{\theta^*}$, we calculate the predictor's distribution for our new data point using Eq. \eqref{eq:truncate}. 
Note that the predictor's parametric distribution also depends on the group membership of the example. 
For fairness constraints 
not based on the actual label $Y$,
e.g., \textsc{D.P.}, this parametric distribution can be directly applied to make predictions.
However, for fairness constraints that depend on the true label, e.g., \textsc{E.Opp.} and \textsc{E.Odd.}, 
we introduce a prediction procedure that estimates the true label using the approximator's parametric distribution. 

For fairness constraints that depend on the true label, our algorithm outputs the predictor and approximator's parametric distributions conditioned on the value of true label, i.e., 
$\Pbb(\yhat|{\bf x},a,y)$ and $\Qbb(\yhat|{\bf x},a,y)$. 
Our goal is to produce the conditional probability of $\yhat$ that does not depend on the true label, i.e., $\Pbb(\yhat|{\bf x},a)$. We construct the following procedure to estimate this probability. 
Based on the marginal probability rule, $\Pbb(\yhat|\xvec,a)$ can be expressed as: 
\begin{align}
     \Pbb(\yhat|{\bf x},a) &= \Pbb(\yhat|{\bf x},a,y=1) P(y=1|{\bf x},a) \label{eq:pred-fp} \\ & \quad 
     + \Pbb(\yhat|\xvec,a,y=0) P(y=0|\xvec,a). \notag
\end{align}
However, since we do not have access to $P(y|{\bf x},a)$, we cannot directly apply this expression. Instead, we approximate $P(y|{\bf x},a)$ with the approximator's distribution $\Qbb(\yhat|{\bf x},a)$. Using the similar marginal probability rule, we express the estimate as:
\begin{align}
    \Qbb(\yhat|\xvec,a) &\approx \Qbb(\yhat|\xvec,a,y=1) \Qbb(\yhat=1|\xvec,a) \\ & \quad 
     + \Qbb(\yhat|\xvec,a,y=0) \Qbb(\yhat=0|\xvec,a). \notag
\end{align}
By rearranging the terms above, we calculate the estimate as:
\begin{align}
    \Qbb(\yhat\!=\!1|\xvec,a) &\!=\! \Qbb(\yhat\!=\!1|\xvec,a,y\!=\!0) / (\Qbb(\yhat\!=\!0|\xvec,a,y\!=\!1) \notag \\ & \quad\!+\! \Qbb(\yhat\!=\!1|\xvec,a,y\!=\!0)), \label{eq:pcheck-fp}
\end{align}
which is directly computed from the approximator's parametric distribution produced by our model using Eq. \eqref{eq:pcheck}.
Finally, to obtain the predictor's conditional probability estimate ($\Pbb(\yhat|\xvec,a)$), we replace $P(y|\xvec,a)$ in Eq. \eqref{eq:pred-fp} with $\Qbb(\yhat|\xvec,a)$ calculated from Eq. \eqref{eq:pcheck-fp}.

\subsection{Asymptotic convergence property}

The ideal behavior of an algorithm is an important consideration in its design.
Asymptotic convergence properties consider 
a learning algorithm when it is provided with access to the population distribution $P(\xvec,a,y)$ and a fully expressive feature representation. 
We show in Theorem \ref{thm:consistency} that in the limit, our method finds a predictor distribution that has a desirable characteristic in terms of the Kullback-Leibler (KL) divergence from the true distribution. 

\begin{theorem}
\label{thm:consistency}
Given 
the population distribution $P(\xvec,a,y)$ and a fully expressive feature representation, our formulation (Def.  \ref{def:fairRobust}) finds the {\bf fair} predictor with the minimal KL-divergence from $P(\xvec,a,y)$.
\end{theorem}

We next show in Theorem \ref{thm:consistency2} that for the case where the fairness constraint depends on the true label (e.g., \textsc{E.Opp.} and \textsc{E.Odd.}), our prediction procedure outputs a predictor distribution with the same desired characteristic, after being marginalized over the true label.

\begin{theorem}
\label{thm:consistency2}
For fairness constraints that depend on the true label, our inference procedure in Eq. \eqref{eq:pred-fp} produces the marginal predicting distribution $\mathbb{P}$ of the fair predictor distribution with the closest KL-divergence to $P(\xvec,a,y)$ in the limit.
\end{theorem}

\section{Experiments}

\subsection{Illustrative behavior on synthetic data}

\begin{figure*}
{\small 
\hspace{-1cm}
$\Pbb(\Yhat=1|A=1)$
\hspace{5cm}
$\Pbb(\Yhat=1|A=0)$
\hspace{1cm}}
\centering
\includegraphics[width=.95\textwidth]{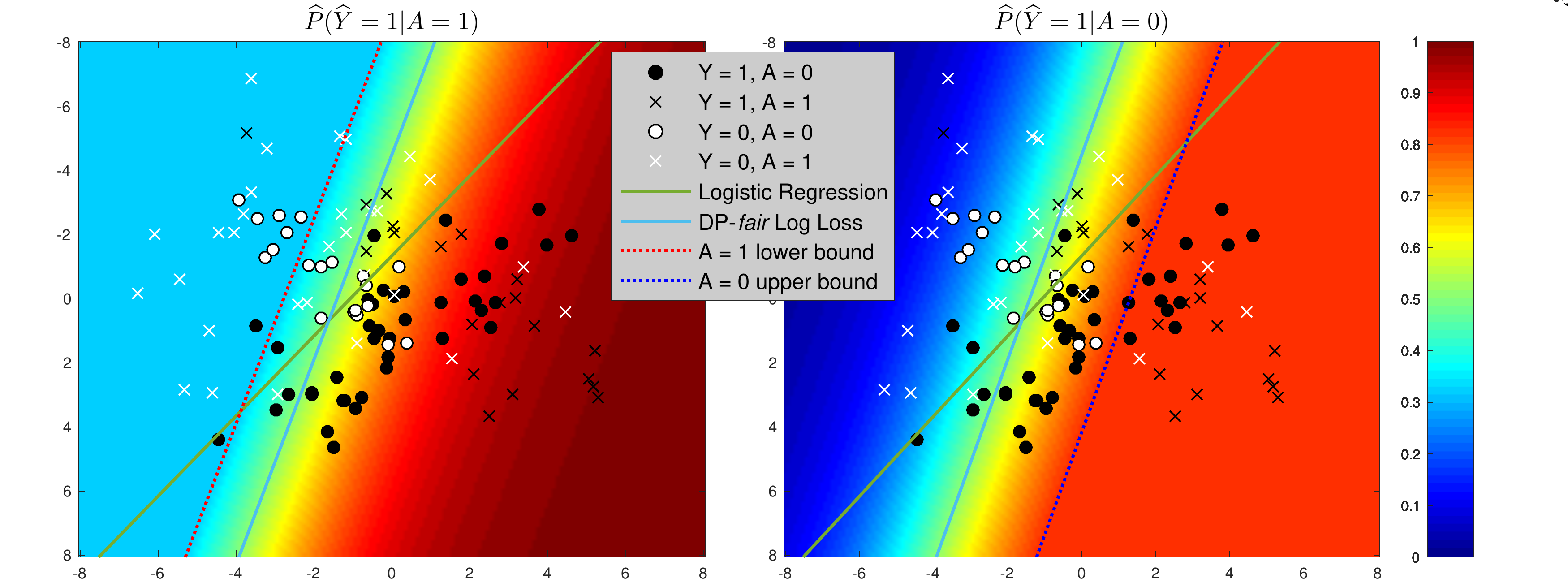}
     \caption{Experimental results on a synthetic dataset with: a heatmap indicating the predictive probabilities of our approach, along with decision and threshold boundaries; and the unfair logistic regression decision boundary. }
\label{fig:fair-predictions}
\end{figure*}

We illustrate the key differences between our model and logistic regression with demographic parity requirements on 2D synthetic data in Figure \ref{fig:fair-predictions}.
The predictive distribution includes different truncated probabilities for each group: raising the minimum probability for group $A=1$ and lowering the maximum probability for group $A=0$.
This permits a decision boundary that differs significantly from the logistic regression decision boundary and better realizes the desired fairness guarantees. 
In contrast, \emph{post-processing methods} using logistic regression as the base classifier \cite{hardt2016equality} are constrained to reshape the given unfair logistic regression predictions without shifting the decision boundary orientation, often leading to suboptimality \cite{woodworth2017learning}.

\subsection{Datasets}
We evaluate our proposed algorithm on three benchmark fairness datasets: 
\begin{enumerate}[leftmargin=*,itemsep=1pt,topsep=0pt,label=(\arabic*),itemindent=0pt]
\item
The \textbf{UCI Adult} \cite{Uci2017} dataset 
includes 45,222 samples with an income greater than \$50k considered to be a favorable binary outcome. We choose gender as the protected attribute, leaving 11 other features 
for each example.
\item
The ProPublica's \textbf{COMPAS} recidivism dataset 
\cite{larson2016we} contains 6,167 samples, and the task is to predict the recidivism of an individual based on criminal history, with the binary protected attribute being race (white and non-white) and an
additional nine features. 
\item The dataset from the Law School Admissions Council’s National Longitudinal Bar Passage Study \cite{Wightman98} has 20,649
examples. Here, the favorable outcome for the individual is passing the bar exam, with race (restricted to
white and black only) as the protected attribute, and 13 other features.
\end{enumerate}

\subsection{Comparison methods}

We compare our method (\emph{Fair Log-loss}) against various baseline/fair learning algorithms that are primarily based on logistic regression as the base classifier:
\begin{enumerate}[leftmargin=*,itemsep=0pt,topsep=0pt,label=(\arabic*),itemindent=0pt]
\item {\bf Unconstrained logistic regression} 
is a standard logistic regression model that ignores all fairness requirements. 
    \item The {\bf cost sensitive reduction approach} 
    by \citet{agarwal2018reductions} reduces fair classification to learning a randomized hypothesis over a sequence of cost-sensitive classifiers. We use the sample-weighted implementation of Logistic Regression in scikit-learn 
    as the base classifier, to compare the effect of the reduction approach. 
We evaluate the performance of the model by varying the constraint bounds across the set $\epsilon \in \{.001,.01,.1\}$. 

\item The {\bf constraint-based learning method}\footnote{\url{https://github.com/mbilalzafar/fair-classification}} of 
\cite{zafar2017aistats,zafar2017fairness} 
uses a covariance 
proxy measure to achieve equalized odds (under the name disparate mistreatment) \cite{zafar2017fairness}, and improve the disparate impact ratio
\cite{zafar2017aistats}, which we use as a baseline method to evaluate demographic parity violation. They cast the resulting non-convex optimization as a
disciplined convex-concave program in training time.
We use the logistic regression as the base classifier.

\item For demographic parity, we compare with the {\bf reweighting method} (\emph{reweighting}) of \citet{Kamiran2012}, which learns weights for each combination of class label and protected attribute and then uses these weights to resample from the original training data which yields a new dataset with no statistical dependence between class label and protected attribute. The new balanced dataset is then used for training a classifier. We use IBM AIF360 toolkit 
to run this method. 


\item For equalized odds, we also compare with the {\bf post-processing method} 
of \citet{hardt2016equality} which transforms the classifier's output by solving a linear program that finds a 
prediction minimizing misclassification errors and satisfying the equalized odds constraint from the set of probability formed by the convex hull of the original classifier's probabilities and the extreme point of probability values (i.e., zero and one).
\end{enumerate}
\begin{figure*}[t!]
    \begin{tabular}{c c c}
    \includegraphics[width=.30\textwidth]{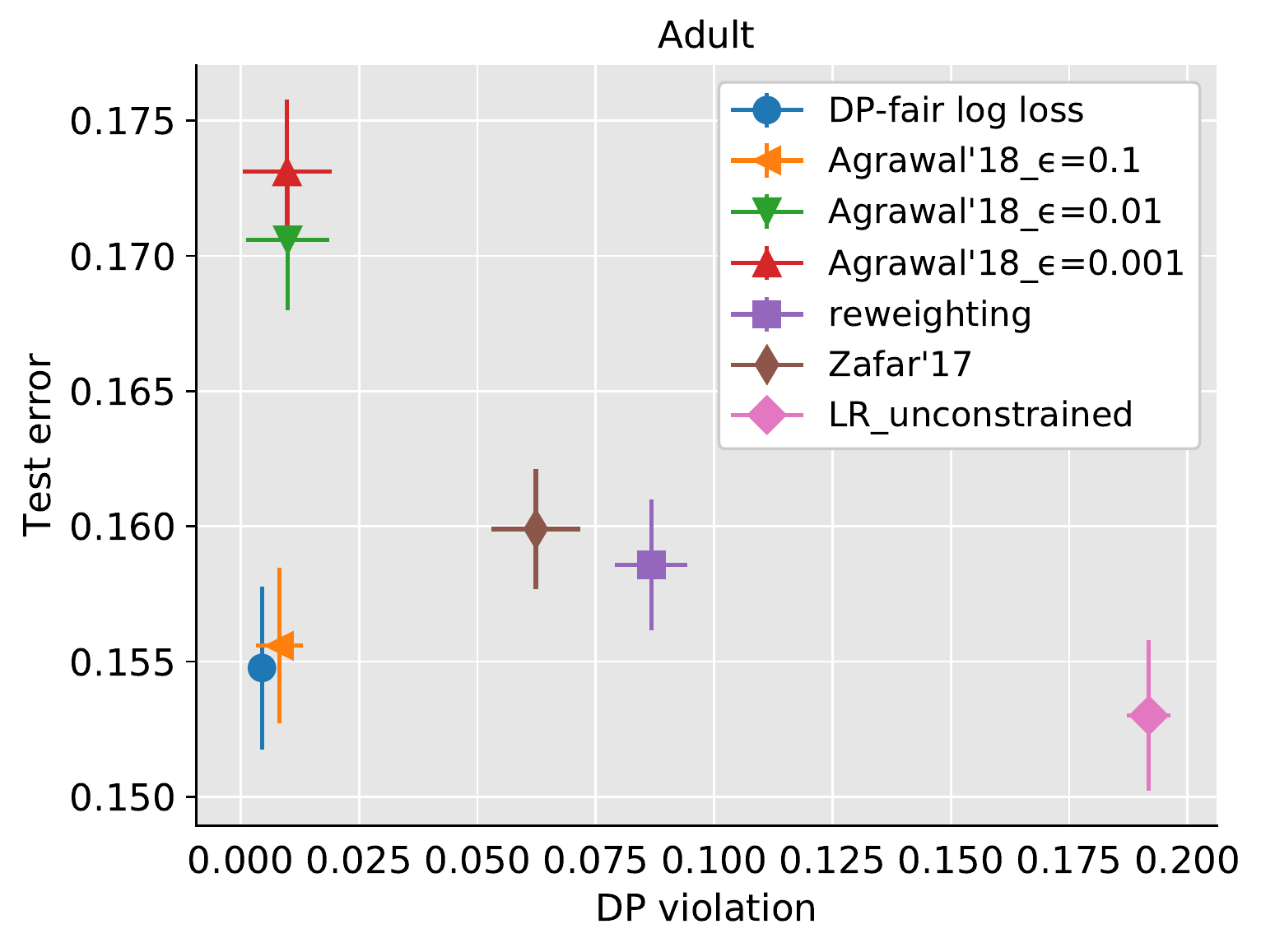}&
    \includegraphics[width=.30\textwidth]{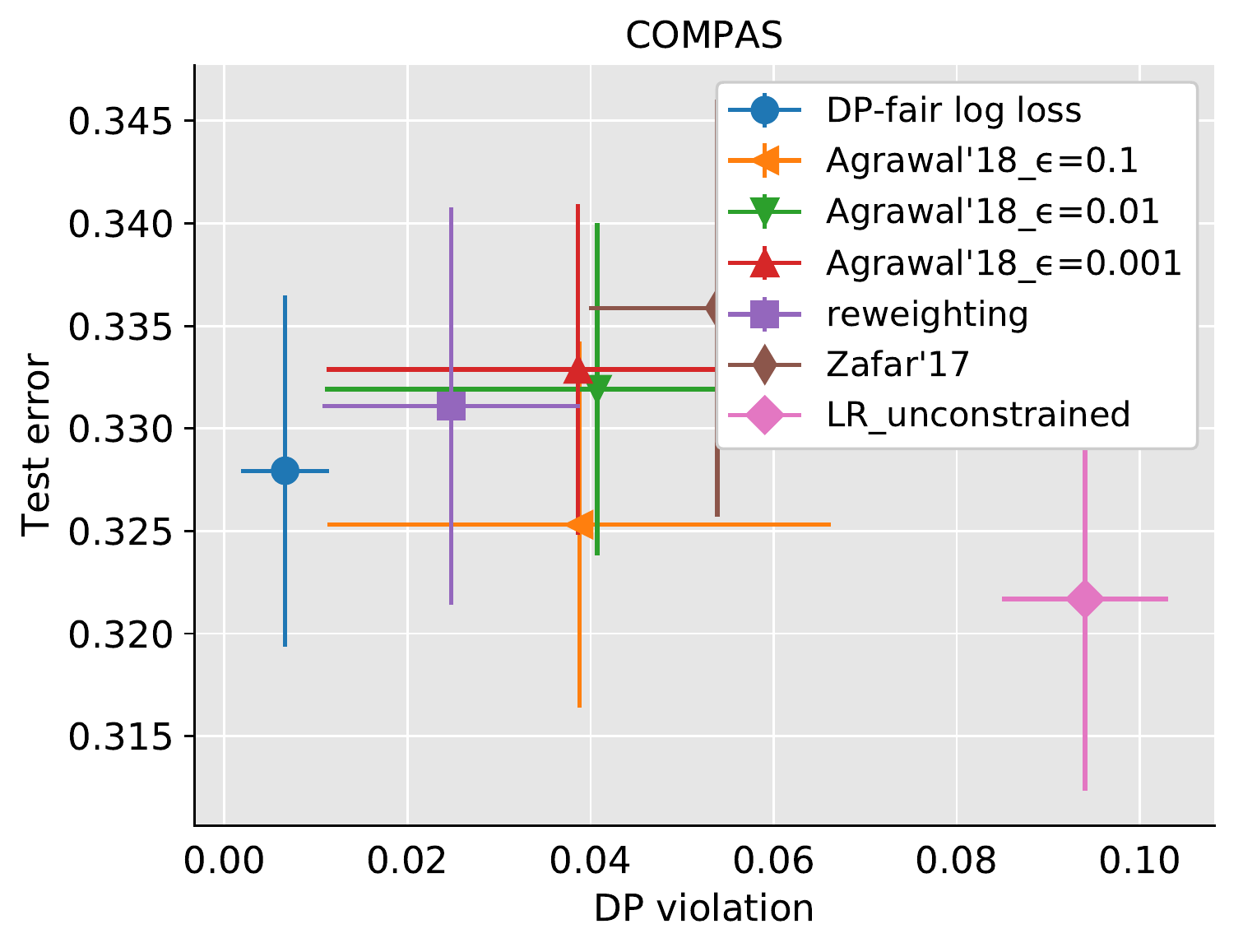}&
    \includegraphics[width=.30\textwidth]{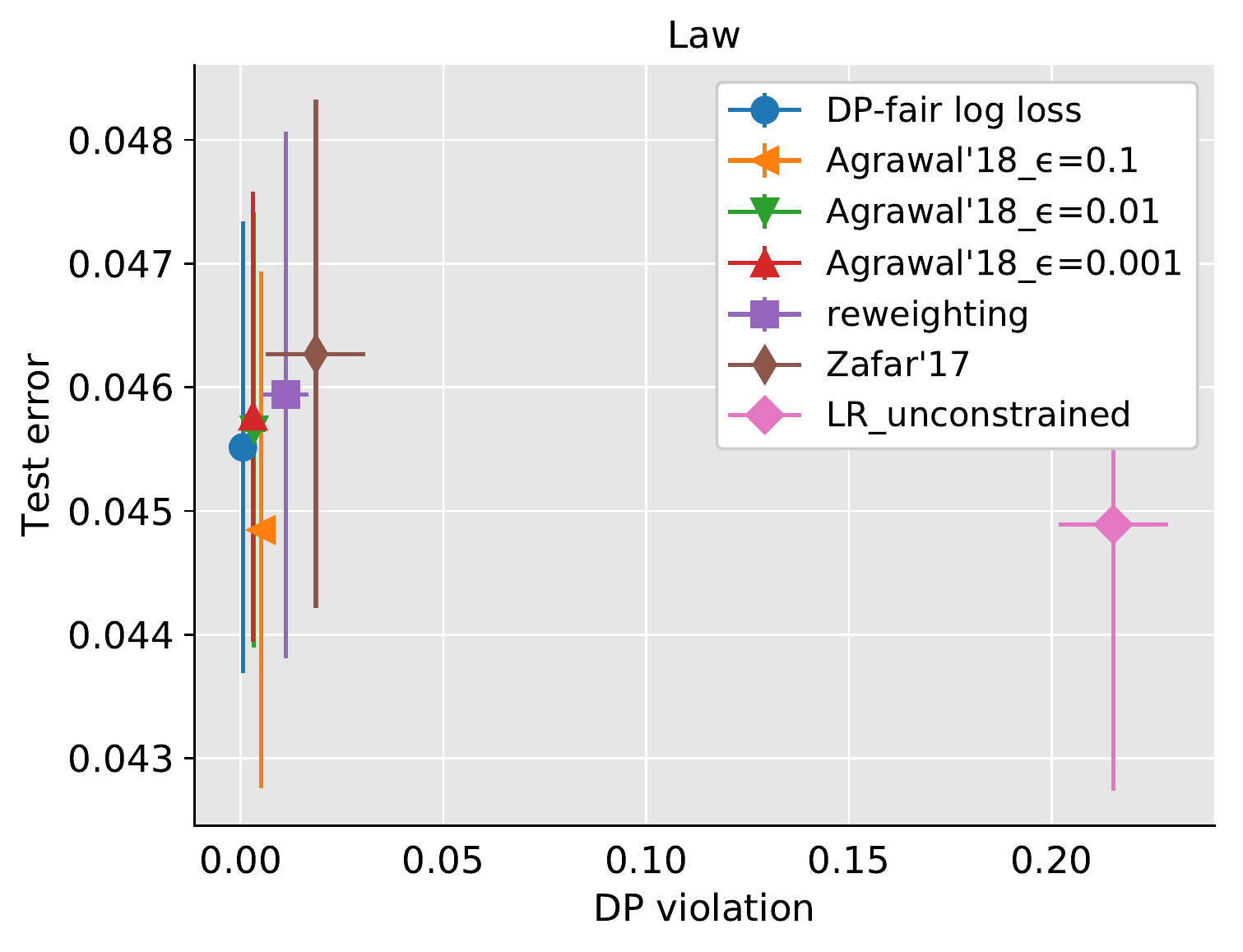}\\
    \includegraphics[width=.30\textwidth]{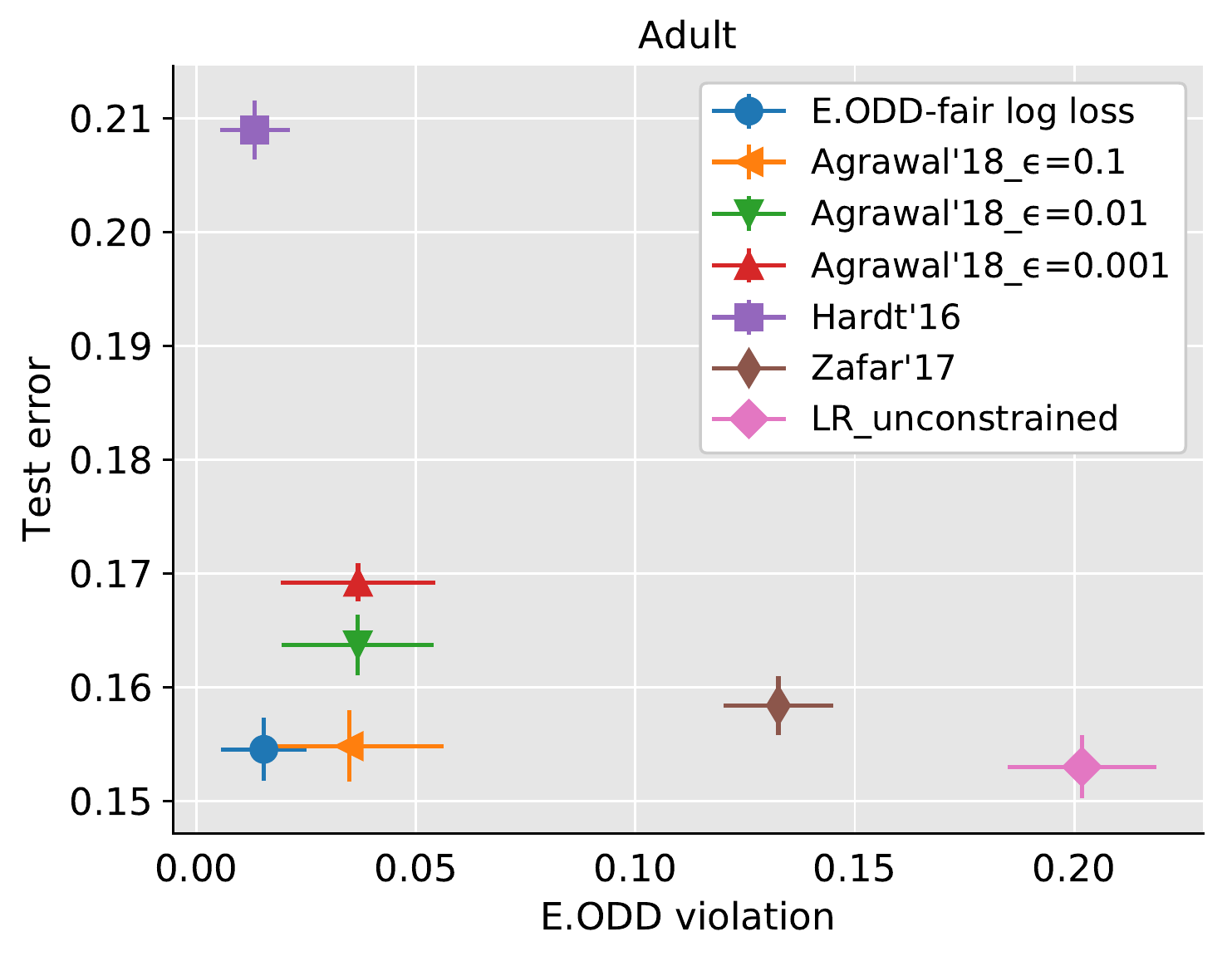}&
    \includegraphics[width=.30\textwidth]{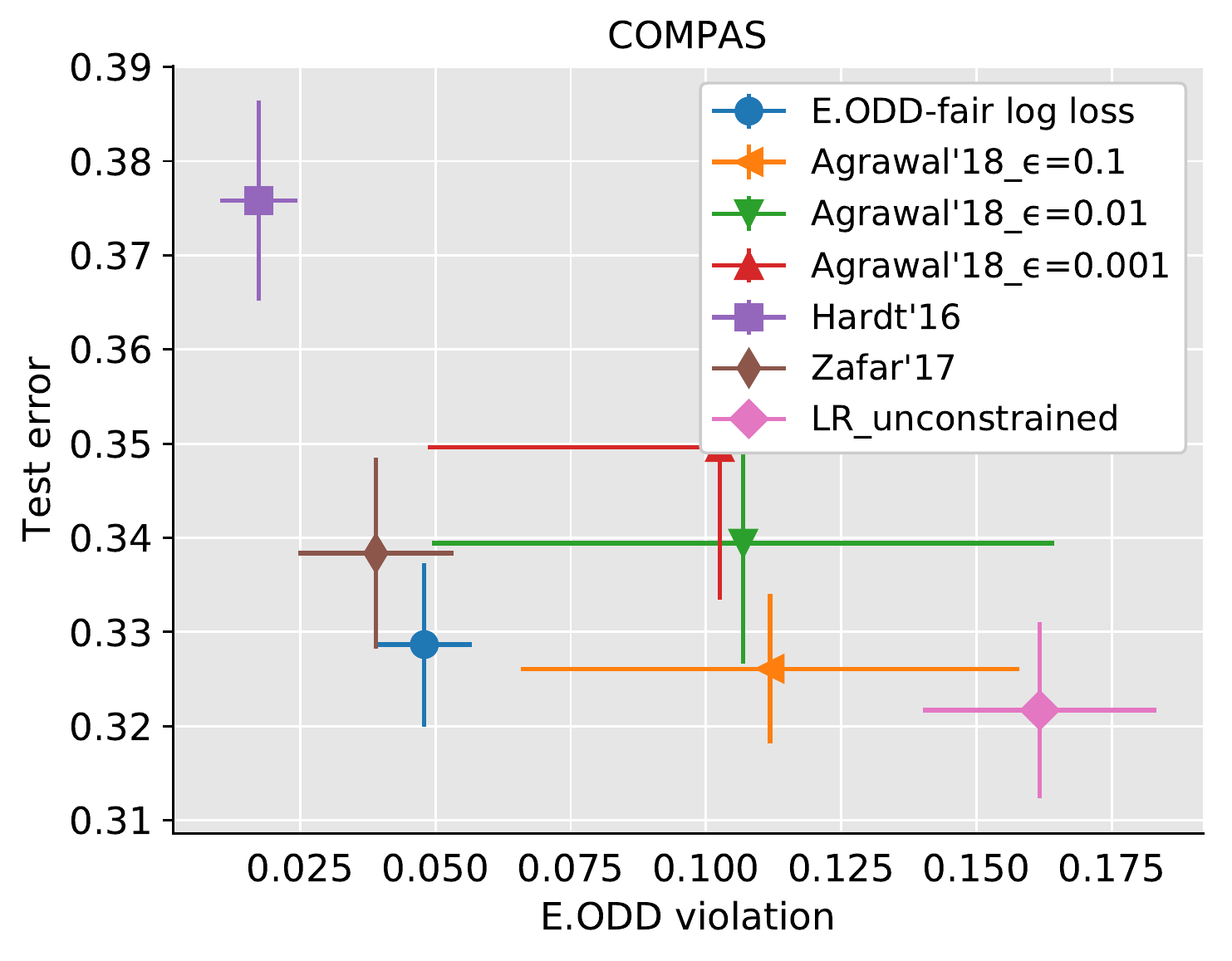}&
    \includegraphics[width=.30\textwidth]{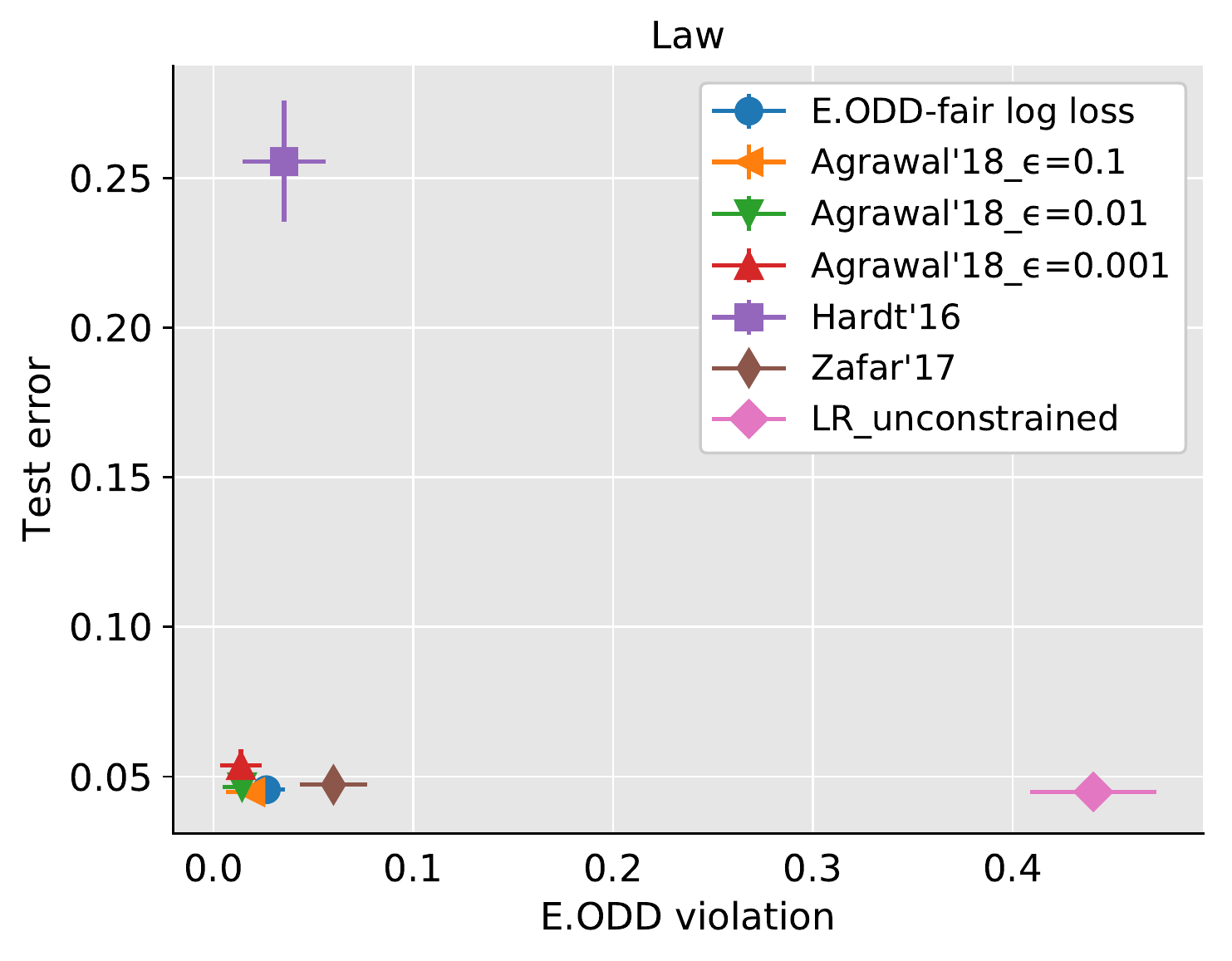}
    \end{tabular}
    \caption{\emph{Test classification error} versus \emph{Demographic Parity}  (top row) and \emph{Equalized Odds} (bottom row) constraint violations. The bars indicate standard deviation on 20 random splits of data.}
    \label{fig:test_result}
\end{figure*}

\subsection{Evaluation measures and setup}

Data-driven fair decision methods seek to minimize both prediction error rates and measures of unfairness.
We consider the misclassification rate (i.e., the 0-1 loss, $\mathbb{E}[\Yhat\neq Y]$) on a withheld test sample to measure prediction error.
To quantify the unfairness of each method, we measure the degree of fairness violation for demographic parity (\textsc{D.P.}) as:
$
\big|\mathbb{E}
[\Ibb(\Yhat=1) | A = 1] - \mathbb{E}[\Ibb(\Yhat=1)|A = 0]\big|,
$
and the sum of fairness violations for each class to measure the total violation for equalized odds (\textsc{E.Odd.}) as:
$
\sum_{y \in \{0,1\}} \big( \big|\mathbb{E} [\Ibb(\Yhat=1) |A = 1, Y = y] - 
\mathbb{E} [\Ibb(\Yhat=1) |A = 0, Y = y] \big|\big),
$
to obtain a level comparison across different methods.
We follow the methodology of \citet{agarwal2018reductions} to give all methods access to the protected attribute both at training and testing time by including the protected attribute in the feature vector. 
We perform all of our experiments using 20 random splits of each dataset into a training set (70\% of examples) and a testing set (30\%). We record the averages over these twenty random splits and the standard deviation.
We cross validate our model on a separate validation set using the best logloss to select an L2 penalty from ($\{.001, .005, .01,.05, .1, .2, .3, .4, .5\}$).

\subsection{Experimental Results}

Figure~\ref{fig:test_result} provides the evaluation results (test error and fairness violation) of each method for 
demographic parity and equalized odds on test data from each of the three datasets 
Fairness can be vacuously achieved by an agnostic predictor that always outputs labels according to independent (biased) coin flips.
Thus, the appropriate question to ask when considering these results is: ``how much additional test error is incurred compared to the baseline of the unfair logistic regression model for how much of an increase in fairness?''

For demographic parity on the Adult dataset, our \emph{Fair Log-loss} approach outperforms all baseline methods on average for both test error rate and for fairness violation, and on COMPAS dataset it achieves the lowest ratio of increased fairness over increased error.
Additionally, the increase in test error over the unfair unconstrained logistic regression model is small.
For demographic parity on the Law dataset, the relationship between methods is not as clear, but our \emph{Fair Log-loss} approach still resides in the Pareto optimal set, i.e., there are no other methods that are significantly better than our result on both criteria.
For equalized odds, \emph{Fair Log-loss} provides the lowest ratios of increased fairness over increased error rate for the Adult and COMPAS datasets, and competitive performance on the Law dataset.
The post-processing method provides comparable or better fairness at the cost of significantly higher error rates. This shows that the approximation in our prediction procedure 
does not significantly impact the performance of our method. 
In terms of the running time, 
our method is an order of magnitude faster 
than comparable methods
(e.g.,  
the train and test running time on one random split of the Adult dataset takes approximately 5 seconds by our algorithm, 80 seconds for the constraint-based method \cite{zafar2017aistats}, and 100 seconds for the reduction-based method \citep{agarwal2018reductions}).

\section{Conclusions and Future Work}

We have developed a novel approach for providing fair data-driven decision making in this work by deriving a new classifier from the first principles of distributionally robust estimation \cite{topsoe1979information,grunwald2004game,delage2010distributionally}.
We formulated a learning objective that imposes fairness requirements on the predictor and views uncertainty about the population distribution pessimistically while maintaining a semblance of the training data characteristics through feature-matching constraints. This resulted in a parametric exponential family conditional distribution that resemble a truncated logistic regression model.

In future work, we plan to investigate the setting in which group membership attributes are not available at testing time.
Extending our approach using a plug-in estimator of $P(a|{\bf x})$ in the fairness constraints introduces this estimator in the parametric form of the model.
Understanding the impact of error from this estimator on predictor fairness in both theory and practice is an important direction of research.


\section*{Acknowledgments}

This work was supported, in part, by the National Science Foundation under Grant No. 1652530 and by the Future of Life
Institute (futureoflife.org) FLI-RFP-AI1 program.
{
\small
\bibliographystyle{aaai}
\bibliography{AAAI-RezaeiA.7734}
}


\titleformat{\section}{\large\bfseries}{\appendixname~\thesection .}{0.5em}{}
\appendix
\onecolumn


\noindent {\LARGE Supplementary Materials}

\section{Proofs}

\subsection{Proof of Theorem \ref{thm:parametric}}

\begin{proof}[Proof of Theorem \ref{thm:parametric}]

Following the strong minimax duality analysis in "log-loss game" \citep{topsoe1979information,grunwald2004game}, and strong Lagrange duality in convex optimization \citep{boyd2004convex}, we perform the following transformations:
\begin{align}
    &\min_{\Pbb \in  \Delta \cap \Gamma} \max_{\Qbb \in \Delta \cap \Xi} \Ebb_{
        \substack{\Ptil(\xvec,a,y)\\
            \Pchk(\yhat|\xvec,a,y)} } 
            \left[-\log \Pbb(\Yhat|\Xbf,A,Y)\right] \\
    \overset{(a)}{=}& \max_{\Qbb \in \Delta \cap \Xi} \min_{\Pbb \in  \Delta \cap \Gamma} \Ebb_{
        \substack{\Ptil(\xvec,a,y)\\
            \Pchk(\yhat|\xvec,a,y)} } 
            \left[-\log \Pbb(\Yhat|\Xbf,A,Y)\right] \\
    \overset{(b)}{=}& \max_{\Qbb \in \Delta}
    \min_\theta
    \min_{\Pbb \in  \Delta \cap \Gamma} \Ebb_{
        \substack{\Ptil(\xvec,a,y)\\
            \Pchk(\yhat|\xvec,a,y)} } 
            \left[-\log \Pbb(\Yhat|\Xbf,A,Y)\right] + \theta^\top \Big( \Ebb_{ \substack{ \Ptil(\xvec,a,y) \\
            \Pchk(\yhat|\xvec,a,y)}
    }[\phi(\Xbf,\Yhat)]
    - \Ebb_{\Ptil(\xvec,a,y)}\left[\phi(\Xbf,Y) \right] \Big) \\
    \overset{(c)}{=}& \min_\theta \max_{\Qbb \in \Delta}
    \min_{\Pbb \in  \Delta \cap \Gamma} \Ebb_{
        \substack{\Ptil(\xvec,a,y)\\
            \Pchk(\yhat|\xvec,a,y)} } 
            \left[-\log \Pbb(\Yhat|\Xbf,A,Y)\right] + \theta^\top \Big( \Ebb_{ \substack{ \Ptil(\xvec,a,y) \\
            \Pchk(\yhat|\xvec,a,y)}
    }[\phi(\Xbf,\Yhat)]
    - \Ebb_{\Ptil(\xvec,a,y)}\left[\phi(\Xbf,Y) \right] \Big) \\
    \overset{(d)}{=}& \min_\theta \min_{\Pbb \in  \Delta \cap \Gamma} \max_{\Qbb \in \Delta}
    \Ebb_{
        \substack{\Ptil(\xvec,a,y)\\
            \Pchk(\yhat|\xvec,a,y)} } 
            \left[-\log \Pbb(\Yhat|\Xbf,A,Y)\right] + \theta^\top \Big( \Ebb_{ \substack{ \Ptil(\xvec,a,y) \\
            \Pchk(\yhat|\xvec,a,y)}
    }[\phi(\Xbf,\Yhat)]
    - \Ebb_{\Ptil(\xvec,a,y)}\left[\phi(\Xbf,Y) \right] \Big) \\
    \overset{(e)}{=}& \min_\theta \min_{\Pbb \in  \Delta} \max_\lambda \max_{\Qbb \in \Delta}
    \Ebb_{
        \substack{\Ptil(\xvec,a,y)\\
            \Pchk(\yhat|\xvec,a,y)} } 
            \left[-\log \Pbb(\Yhat|\Xbf,A,Y)\right] + \theta^\top \Big( \Ebb_{ \substack{ \Ptil(\xvec,a,y) \\
            \Pchk(\yhat|\xvec,a,y)}
    }[\phi(\Xbf,\Yhat)]
    - \Ebb_{\Ptil(\xvec,a,y)}\left[\phi(\Xbf,Y) \right] \Big) \\
    & \qquad \qquad + \lambda \Big( \tfrac{1}{p_{\gamma_1}} \Ebb_{ \substack{ \Ptil(\xvec,a,y) \\
            \Pbb(\yhat|\xvec,a,y)}
    }[\Ibb(\Yhat\!=\!1\wedge \gamma_1(A,Y))] 
    - \tfrac{1}{p_{\gamma_0}} \Ebb_{ \substack{ \Ptil(\xvec,a,y) \\
            \Pbb(\yhat|\xvec,a,y)}
    }[\Ibb(\Yhat\!=\!1 \wedge \gamma_0(A,Y))]  \Big) \notag \\
    \overset{(f)}{=}& \min_\theta \max_\lambda \min_{\Pbb \in  \Delta}  \max_{\Qbb \in \Delta}
    \Ebb_{
        \substack{\Ptil(\xvec,a,y)\\
            \Pchk(\yhat|\xvec,a,y)} } 
            \left[-\log \Pbb(\Yhat|\Xbf,A,Y)\right] + \theta^\top \Big( \Ebb_{ \substack{ \Ptil(\xvec,a,y) \\
            \Pchk(\yhat|\xvec,a,y)}
    }[\phi(\Xbf,\Yhat)]
    - \Ebb_{\Ptil(\xvec,a,y)}\left[\phi(\Xbf,Y) \right] \Big) \\
    & \qquad \qquad + \lambda \Big( \tfrac{1}{p_{\gamma_1}} \Ebb_{ \substack{ \Ptil(\xvec,a,y) \\
            \Pbb(\yhat|\xvec,a,y)}
    }[\Ibb(\Yhat\!=\!1\wedge \gamma_1(A,Y))] 
    - \tfrac{1}{p_{\gamma_0}} \Ebb_{ \substack{ \Ptil(\xvec,a,y) \\
            \Pbb(\yhat|\xvec,a,y)}
    }[\Ibb(\Yhat\!=\!1 \wedge \gamma_0(A,Y))]  \Big) \notag \\
    \overset{(g)}{=}& \min_\theta \max_\lambda \min_{\Pbb \in  \Delta}  \max_{\Qbb \in \Delta}
    \Ebb_{\Ptil(\xvec,a,y)} \Bigg[
    \Ebb_{\Pchk(\yhat|\xvec,a,y)} 
            \left[-\log \Pbb(\Yhat|\Xbf,A,Y)\right] + \theta^\top \Big( \Ebb_{ \Pchk(\yhat|\xvec,a,y)} [\phi(\Xbf,\Yhat)]
    - \phi(\Xbf,Y) \Big) \\
    & \qquad \qquad + \lambda \Big( \tfrac{1}{p_{\gamma_1}} \Ebb_{ \Pbb(\yhat|\xvec,a,y)} [\Ibb(\Yhat\!=\!1\wedge \gamma_1(A,Y))] 
    - \tfrac{1}{p_{\gamma_0}} \Ebb_{\Pbb(\yhat|\xvec,a,y)} [\Ibb(\Yhat\!=\!1 \wedge \gamma_0(A,Y))]  \Big) \Bigg] \notag 
\end{align}

The transformation steps above are described as follows:
\begin{enumerate}[label=(\alph*),itemsep=1pt]
    
    \item We flip the min and max order using strong minimax duality in "log-loss game" \citep{topsoe1979information,grunwald2004game}.

    \item We introduce the Lagrange dual variable $\theta$ to directly incorporate the moment matching constraints over $\Pchk$ into the objective function.

    
    \item The objective is concave on $\Pchk$ for all $\theta$, while $\Xi$ is a convex set. Given a feasible solution on the relative interior of $\Xi$, strong Lagrange duality holds \cite{boyd2004convex} and thus we can flip the optimization order of $\Pchk$ and $\theta$.
    
    \item We flip the inner min and max over $\Pbb$ and $\Pchk$ using the minimax duality, as in (a).
    
    \item We introduce the Lagrange dual variable $\lambda$ to directly incorporate the fairness constraints over $\Pbb$ into the objective function.
    
    \item Similar to (c), we use strong Lagrange duality theorem to flip the optimization order of $\lambda$ and $\Pbb$.

    \item We group the expectation with respect to the empirical training data.

\end{enumerate}

We now focus on the inner minimax formulation over $\Pbb$ and $\Pchk$, given the value of $\theta$ and $\lambda$, i.e.:
\begin{align}
    \min_{\Pbb \in  \Delta}  \max_{\Qbb \in \Delta} & \;\;
    \Ebb_{\Ptil(\xvec,a,y)} \Bigg[
    \Ebb_{\Pchk(\yhat|\xvec,a,y)} 
            \left[-\log \Pbb(\Yhat|\Xbf,A,Y)\right] + \theta^\top \Big( \Ebb_{ \Pchk(\yhat|\xvec,a,y)} [\phi(\Xbf,\Yhat)]
    - \phi(\Xbf,Y) \Big) \\
    & \qquad \qquad + \lambda \Big( \tfrac{1}{p_{\gamma_1}} \Ebb_{ \Pbb(\yhat|\xvec,a,y)} [\Ibb(\Yhat\!=\!1\wedge \gamma_1(A,Y))] 
    - \tfrac{1}{p_{\gamma_0}} \Ebb_{\Pbb(\yhat|\xvec,a,y)} [\Ibb(\Yhat\!=\!1 \wedge \gamma_0(A,Y))]  \Big) \Bigg] \notag
\end{align}

We aim to find the analytical solutions for $\Pbb$ and $\Pchk$ in the equation above. First, we write the Lagrangian by incorporating the probability simplex constraints into the objective, i.e.:
\begin{align}
\label{eq:lagrangian}
 \min_{\Pbb} \max_{\Qbb}
\min_{Z_\Qbb,\beta \geq 0}\max_{Z_\Pbb} &\; 
    L(\Pbb,\Qbb,Z_\Pbb,Z_\Qbb,\beta) =   \min_{Z_\Qbb,\beta \geq 0}\max_{Z_\Pbb} 
    \max_{\Qbb} \min_{\Pbb} 
     \mathbb{E}_{\widetilde{P}(\xvec,a,y)\Qbb(\widehat{y}|\xvec,a,y)} [-\log(\Pbb(\widehat{Y}|\Xbf,A,Y))] \\
    + &\theta^\top  (\mathbb{E}_{\widetilde{P}(\xvec,a,y)\Qbb(\widehat{y}|\xvec,a,y)} [\phi(\Xbf,\widehat{Y})] - \mathbb{E}_{\widetilde{P}(\xvec,a,y)} [\phi(\Xbf,Y)]) 
    + \mathbb{E}_{\widetilde{P}(\xvec,a,y)\Pbb(\widehat{y}|\xvec,a,y)} [F_\lambda(A,Y,\widehat{Y})] \notag \\
    + &\sum_{(\xvec,a,y) \in \mathcal{D}} Z_\Pbb(\xvec,a) \left[\mathbb{E}_{\Pbb(\widehat{y}|\xvec,a,y)} [1 | \xvec,a,y] - 1\right] 
    + \sum_{(\xvec,a,y) \in \mathcal{D}} Z_\Qbb(\xvec,a) \left[\mathbb{E}_{\Qbb(\widehat{y}|\xvec,a,y)} [1 | \xvec,a,y] - 1\right] \notag \\ 
    + &\sum_{\xvec,a,y \in \mathcal{D}}\sum_{\widehat{y} \in \mathcal{Y}} \beta(\xvec,a,y,\widehat{y}) \Qbb(\widehat{y}|\xvec,a,y) \notag
\end{align}
\begin{align}
\text{where : } F_\lambda(a,y,\widehat{y}) = 
    \begin{cases}
        \frac{\lambda}{p_{\gamma_1}} & \text{if }  \yhat = 1 \wedge \gamma_1(a,y) \\
        -\frac{\lambda}{p_{\gamma_0}} & \text{if } \yhat = 1 \wedge \gamma_0(a,y) \\
        0 & \text{otherwise}.
    \end{cases}
\end{align}


We now take the derivative of the Lagrangian with respect to $\Pbb(\yhat|\xvec,a,y)$:

\begin{align} \label{eq:derivativeWrtPhat}
\frac{\partial L}{\partial \Pbb(\yhat|\xvec,a,y)} &= -\frac{\widetilde{P}(\xvec,a,y)\Qbb(\yhat|\xvec,a,y)}{\Pbb(\yhat|\xvec,a,y)} + \widetilde{P}(\xvec,a,y)F_\lambda(a,y,\yhat) + Z_\Pbb(x,a).
\end{align}
    
By setting Eq. \eqref{eq:derivativeWrtPhat} to zero, we rewrite $\Pbb$ in terms of $\Qbb$:
\begin{align} 
\Pbb(\yhat|\xvec,a,y) &= \frac{\widetilde{P}(\xvec,a,y)\Qbb(\yhat|\xvec,a,y)}{ \widetilde{P}(\xvec,a,y)F_\lambda(a,y,\yhat) + Z_\Pbb(x,a)} = \frac{\Qbb(\yhat|\xvec,a,y)}{ F_\lambda(a,y,\yhat) + \frac{Z_\Pbb(x,a)}{\Ptil(\xvec,a,y)}}.  \label{eq:phat-in-terms-pcheck}
\end{align}
Using Eq. \eqref{eq:phat-in-terms-pcheck} we rewrite Eq. \eqref{eq:lagrangian} as:
\begin{align}
    L(\Qbb,Z_\Pbb,Z_\Qbb,\beta) =  \min_{Z_\Qbb,\beta \geq 0} & \max_{Z_\Pbb} \max_{\Qbb}   \mathbb{E}_{\widetilde{P}(\xvec,a,y)\Qbb(\widehat{y}|\xvec,a,y)} \left[-\log \Pbb(\widehat{Y}|\Xbf,A,Y)\right]  \\
    + &\theta^\top  (\mathbb{E}_{\widetilde{P}(\xvec,a,y)\Qbb(\widehat{y}|\xvec,a,y)} \left[\phi(\Xbf,\widehat{Y})\right] - \mathbb{E}_{\widetilde{P}(\xvec,a,y)} \left[\phi(\Xbf,Y)\right]) \notag \\
    + & \underbrace{\mathbb{E}_{\widetilde{P}(\xvec,a,y)\Pbb(\widehat{y}|\xvec,a,y)} \left[F_\lambda(A,Y,\widehat{Y})\right]
    + \mathbb{E}_{\widetilde{P}(\xvec,a,y)\Pbb(\widehat{y}|\xvec,a,y)} \left[\frac{Z_\Pbb(\xvec,a)}{\widetilde{P}(\xvec,a,y)}\right]}_{=\mathbb{E}_{\widetilde{P}(\xvec,a,y)\check{P}(\widehat{y}|\xvec,a,y)}[1] = 1 \text{ via Eq. \eqref{eq:phat-in-terms-pcheck}}} - \mathbb{E}_{\widetilde{P}(\xvec,a,y)}\left[\frac{Z_\Pbb(\xvec,a)}{\widetilde{P}(\xvec,a,y)}\right] \notag \\
    + &\mathbb{E}_{\widetilde{P}(\xvec,a,y)\Qbb(\widehat{y}|\xvec,a,y)} \left[\frac{Z_\Qbb(\xvec,a)}{\widetilde{P}(\xvec,a,y)}\right] - \mathbb{E}_{\widetilde{P}(\xvec,a,y)}\left[\frac{Z_\Qbb(\xvec,a)}{\widetilde{P}(\xvec,a,y)}\right] 
    + \mathbb{E}_{\widetilde{P}(\xvec,a,y)\Qbb(\widehat{y}|\xvec,a,y)} \left[\frac{\beta(\xvec,a,y,\widehat{y})}{\widetilde{P}(\xvec,a,y)}\right]. \notag
\end{align}
Replacing $\Pbb$ in Lagrangian we get:
\begin{align}
    L(\Qbb,Z_\Pbb,Z_\Qbb,\beta) =  \min_{Z_\Qbb,\beta \geq 0} & \max_{Z_\Pbb} \max_{\Qbb}   \mathbb{E}_{\widetilde{P}(\xvec,a,y)\Qbb(\widehat{y}|\xvec,a,y)} \left[-\log \Qbb(\widehat{Y}|\Xbf,A,Y) 
    + \log\left(F_\lambda(a,y,\yhat) + \frac{Z_\Pbb(x,a)}{\widetilde{P}(\xvec,a,y)}\right)\right]  \\
    + &\theta^\top  (\mathbb{E}_{\widetilde{P}(\xvec,a,y)\Qbb(\widehat{y}|\xvec,a,y)} \left[\phi(\Xbf,\widehat{Y})\right] - \mathbb{E}_{\widetilde{P}(\xvec,a,y)} \left[\phi(\Xbf,Y)\right]) + 1 - \mathbb{E}_{\widetilde{P}(\xvec,a,y)}\left[\frac{Z_\Pbb(\xvec,a)}{\widetilde{P}(\xvec,a,y)}\right] \notag \\
    + &\mathbb{E}_{\widetilde{P}(\xvec,a,y)\Qbb(\widehat{y}|\xvec,a,y)} \left[\frac{Z_\Qbb(\xvec,a)}{\widetilde{P}(\xvec,a,y)}\right] - \mathbb{E}_{\widetilde{P}(\xvec,a,y)}\left[\frac{Z_\Qbb(\xvec,a)}{\widetilde{P}(\xvec,a,y)}\right] 
    + \mathbb{E}_{\widetilde{P}(\xvec,a,y)\Qbb(\widehat{y}|\xvec,a,y)} \left[\frac{\beta(\xvec,a,y,\widehat{y})}{\widetilde{P}(\xvec,a,y)}\right]. \notag
\end{align}

We now calculate the derivative with respect to $\Qbb$. 
\begin{align}
\label{partialPCheck}
 \frac{\partial L}{\partial \Qbb(\widehat{y}|\xvec,a,y)} 
 = & \widetilde{P}(\xvec,a,y)\Bigg(-\log \Qbb(\widehat{y}|\xvec,a,y) -1 + \theta^\top \phi(\xvec,\yhat) + \log\left(F_\lambda(a,y,\yhat) + \frac{Z_\Pbb(\xvec_i,a)}{\widetilde{P}(\xvec,a,y)}\right) \\
    & \qquad \qquad+
    \frac{Z_\Qbb(\xvec,a)}{\widetilde{P}(\xvec,a,y)} + \frac{\beta(\xvec,a,y,\widehat{y})}{\widetilde{P}(\xvec,a,y)} \Bigg) \notag
\end{align}

Setting Eq. \eqref{partialPCheck} to 0 yields:
\begin{align}
    \log \frac{\Qbb(\widehat{y}|\xvec,a,y)}{F_\lambda(a,y,\yhat) + \frac{Z_\Pbb(\xvec_i,a)}{\widetilde{P}(\xvec,a,y)}} &= \theta^\top \phi(\xvec,\yhat) +
    \frac{Z_\Qbb(\xvec,a)}{\widetilde{P}(\xvec,a,y)} +
    \frac{\beta(\xvec,a,y,\widehat{y})}{\widetilde{P}(\xvec,a,y)} - 1 \\
 \Pbb(\widehat{y}|\xvec,a,y) &= e^{\theta^\top \phi(\xvec,\yhat) + \frac{Z_\Qbb(\xvec,a)}{\widetilde{P}(\xvec,a,y)} + \frac{\beta(\xvec,a,y,\yhat)}{\widetilde{P}(\xvec,a,y)} -1}.
\end{align}
We analytically solve the normalization constraint for $\Pbb$, i.e., $\sum_{\widehat{y} \in \mathcal{Y}} \Pbb(\yhat|\xvec,a,y) = 1 $ 
\begin{align}
\label{eq:normalization}
 \sum_{\widehat{y} \in \mathcal{Y}} e^{\theta^\top \phi(\xvec,\yhat)
    + \frac{Z_\Qbb(\xvec,a)}{\widetilde{P}(\xvec,a,y)} + \frac{\beta(\xvec,a,y,\yhat)}{\widetilde{P}(\xvec,a,y)} -1 } = 1 \\
    \frac{Z_\Qbb(\xvec,a)}{\widetilde{P}(\xvec,a,y)} -1 = -\log\sum_{\yhat \in \mathcal{Y}} e^{\theta^\top \phi(\xvec,\yhat) + \frac{\beta(\xvec,a,y,\yhat)}{\widetilde{P}(\xvec,a,y)} },
\end{align}
which yields following parametric form of the predictor distribution:
\begin{align}
\Pbb(\widehat{y}|\xvec,a,y) = \frac{e^{\theta^\top \phi(\xvec,\yhat) + \frac{\beta(\xvec,a,y,\widehat{y})}{\widetilde{P}(\xvec,a,y)} }}{Z_{\theta}(\xvec,a,y)} = \frac{e^{\theta^\top \phi(\xvec,\widehat{y}) + \frac{\beta(\xvec,a,y,\widehat{y})}{\widetilde{P}(\xvec,a,y)} }}{\sum_{y' \in \mathcal{Y}} {e^{\theta^\top \phi(\xvec,y') + \frac{\beta(\xvec,a,y,y')}{\widetilde{P}(\xvec,a,y)} }}}.
\end{align} 

Notice the similarity to standard logistic regression. Where in contrast, here the probability for each class is adjusted with terms $\frac{\beta(\xvec,a,y,\yhat)}{\widetilde{P}(\xvec,a,y)}$ to satisfy the fairness constraints.

From Eq. \eqref{eq:phat-in-terms-pcheck} we get the relation of $\Pchk$ and $\Pbb$. Solving the normalization constraint for $\Qbb(\widehat{y}|\xvec,a,y)$ yields:
\begin{align}
  \quad \sum_{\widehat{y} \in \mathcal{Y}} \Qbb(\widehat{y}|\xvec,a,y) &= 1 \\
 \sum_{\widehat{y} \in \mathcal{Y}} \Pbb(\widehat{y}|\xvec,a,y) \left(F_\lambda(a,y,\widehat{y}) + \frac{Z_\Pbb(\xvec,a)}{\widetilde{P}(\xvec,a,y)} \right) &= 1 \\
 \frac{Z_\Pbb(\xvec,a)}{\widetilde{P}(\xvec,a,y)} &= 1 - \sum_{\widehat{y} \in \mathcal{Y}} \Pbb(\widehat{y}|\xvec,a,y)F_\lambda(a,y,\widehat{y})
\end{align}

Thus, we can rewrite $\Qbb$ as:
\begin{align}
    \Qbb(\widehat{y}|\xvec,a,y) 
    &= \Pbb(\widehat{y}|\xvec,a,y)(F_\lambda(a,y,\widehat{y}) + 1 - \sum_{y' \in \mathcal{Y}} \Pbb(y'|\xvec,a,y)F_\lambda(a,y,y'))
\end{align}

We consider the binary classification $\widehat{y},y=\{0,1\}$, and expand $\Qbb$ as:
\begin{align}
    \Qbb(\yhat=1|\xvec,a,y) 
    &= \Pbb(\yhat=1|\xvec,a,y)[F_\lambda(a,y,1) + 1 - \Pbb(\yhat=1|\xvec,a,y)F_\lambda(a,y,1))] \\
    &= \Pbb(\yhat=1|\xvec,a,y)(1 + \Pbb(\yhat=0|\xvec,a,y)F_\lambda(a,y,1)) \\
    &= \begin{cases}
    \Pbb(\yhat=1|\xvec,a,y)(1 + \frac{\lambda}{p_{\gamma_1}}\Pbb(\yhat=0|\xvec,a,y))    & \text{if } \gamma_1(a,y)\\
    \Pbb(\yhat=1|\xvec,a,y)(1 - \frac{\lambda}{p_{\gamma_0}}\Pbb(\yhat=0|\xvec,a,y))    & \text{if } \gamma_0(a,y) \\
    \Pbb(\yhat=1|\xvec,a,y) & \text{otherwise.}
    \end{cases}
\end{align}

The above equation shows that the approximator's distribution $\Qbb$ is a quadratic function of predictor $\Pbb$, for example in the case where $\gamma_1(a,y)=1$: 
\begin{align}
    \Pchk_{\theta,\lambda}(\yhat=1|\xvec,a,y) =& \rho (1 + \tfrac{\lambda}{p_{\gamma_1}}(1-\rho) )
        = (1 +\tfrac{\lambda}{p_{\gamma_1}}) \rho - \tfrac{\lambda}{p_{\gamma_1}} \rho^2, \notag
\end{align}
where $\rho = \Pbb_{\theta,\lambda}(\yhat\!=\!1|\xvec,a,y)$. 
For the region where the function goes above 1 (or below 1 depending on sign of $F_\lambda$), the predictor's probability must be truncated in terms of fairness function such that $\Qbb = 1$ (or zero). We derive these cases in the following by considering that the complementary slackness ensures non-negativity of $\Qbb$. 

The complementary slackness from the KKT condition requires:
\begin{equation}
    \forall {\xvec,a,y,\widehat{y}}, \quad \beta(\xvec,a,y,\widehat{y})\Qbb(\widehat{y}|\xvec,a,y) = 0.
\end{equation}

Suppose that $\Qbb(\widehat{y}|\xvec,a,y) = 0$, then:
\begin{align}
    \Qbb(\widehat{y}|\xvec,a,y) =& \Pbb(\widehat{y}|\xvec,a,y)\left(F_\lambda(a,y,\yhat) + 1 - \sum_{\bar{y} \in \mathcal{Y}} \Pbb(\bar{y}|\xvec,a,y)F_\lambda(a,y,\bar{y})\right) = 0 \\ 
    \Pbb > 0 \implies & F_\lambda(a,y,\widehat{y}) + 1 - \sum_{\bar{y} \in \{0,1\}} \Pbb(\bar{y}|\xvec,a,y)F_\lambda(a,y,\bar{y}) = 0\\
    & F_\lambda(a,y,\widehat{y}) + 1 - \Pbb(0|\xvec,a,y)F_\lambda(a,y,0) - \Pbb(1|\xvec,a,y)F_\lambda(a,y,1) = 0. 
\end{align}

Since $F_\lambda(a,y,0) = 0$, then the equation above reduces to:
\begin{align}
    F_\lambda(a,y,\widehat{y}) + 1 - \Pbb(\yhat=1|\xvec,a,y)F_\lambda(a,y,1) &= 0 \\
    \Pbb(\yhat=1|\xvec,a,y) = \frac{F_\lambda(a,y,1) + 
    1}{F_\lambda(a,y,1)}.
\end{align}

Observe that the above equation can only hold if $\gamma_1(a,y)  = 1$, or $\gamma_0(a,y) = 1$. For the other cases, complementary slackness requires that $\beta(\xvec,a,y,\widehat{y}) = 0$ and $\Qbb(\widehat{y}|\xvec,a,y) = \Pbb(\widehat{y}|\xvec,a,y) =  \frac{e^{\theta^\top \phi(\xvec,\widehat{y})}}{Z_{\theta}(\xvec)}$.

Thus, we have the following cases:
\begin{align}
    \Pbb(\widehat{y}=1|\xvec,a,y) = 
    \begin{cases}
        \frac{p_{\gamma_1}}{\lambda} & \text{if }  \gamma_1(a,y) \land \check{P}(1|\xvec,a,y) = 1,  \\
        -\frac{p_{\gamma_0}}{\lambda} & \text{if }  \gamma_0(a,y) \land \check{P}(1|\xvec,a,y) = 1 \\
        1 + \frac{p_{\gamma_1}}{\lambda} & \text{if }  \gamma_1(a,y) \land \check{P}(1|\xvec,a,y) = 0 \\
        1 - \frac{p_{\gamma_0}}{\lambda} & \text{if }  \gamma_0(a,y)\land \check{P}(1|\xvec,a,y) = 0 \\
        \frac{e^{\theta^\top \phi(\xvec,1)}}{Z_{\theta}(\xvec)} & \text{otherwise.} 
    \end{cases}
\end{align}

Therefore, if $\lambda \geq 0$, we have the following parametric form for the predictor distribution:
\begin{align}
    \Pbb(\widehat{y}=1|\xvec,a,y) = 
    \begin{cases}
        \min{\{\frac{p_{\gamma_1}}{\lambda},\frac{e^{\theta^\top \phi(\xvec,1)}}{Z_{\theta}(\xvec)}\}} & \text{if } \gamma_1(a,y)\\
        \max{\{1-\frac{p_{\gamma_0}}{\lambda},\frac{e^{\theta^\top \phi(\xvec,1)}}{Z_{\theta}(\xvec)} \}} & \text{if } \gamma_0(a,y) \\
        \frac{e^{\theta\phi(\xvec,1)}}{Z_{\theta}(\xvec)} & \text{otherwise} 
    \end{cases}
\end{align}
and if $\lambda < 0$:
\begin{align}
    \Pbb(\widehat{y}=1|\xvec,a,y) = 
    \begin{cases}
        \max{\{1+\frac{p_{\gamma_1}}{\lambda},\frac{e^{\theta^\top \phi(\xvec,1)}}{Z_{\theta}(\xvec)}\}} & \text{if } \gamma_1(a,y) \\
        \min{\{-\frac{p_{\gamma_1}}{\lambda},\frac{e^{\theta^\top \phi(\xvec,1)}}{Z_{\theta}(\xvec)} \}} & \text{if } \gamma_0(a,y)\\
        \frac{e^{\theta\phi(\xvec,1)}}{Z_{\theta}(\xvec)} & \text{otherwise.} 
    \end{cases}
\end{align}
Note that if $\lambda = 0$, all of the cases collapse to a single case $\Pbb(\widehat{y}=1|\xvec,a,y) = \frac{e^{\theta\phi(\xvec,1)}}{Z_{\theta}(\xvec)}$.

\end{proof}

\subsection{Proof of Theorem \ref{thm:objective}}

\begin{proof}[Proof of Theorem \ref{thm:objective}]

Given the optimum $\lambda^*_\theta$ for each $\theta$, Eq. \eqref{eq:dual} reduces to:
\begin{align}
&\min_{\theta} \; \tfrac{1}{n}\!\!\!\!\!\!
    \sum_{(\xvec,a,y) \in \Dcal} \left\{
    \Ebb_{
        \Pchk_{\theta,\lambda^*_\theta} (\yhat|\xvec,a,y) } 
            \left[-\log \Pbb_{\theta,\lambda^*_\theta}(\Yhat|\xvec,a,y)\right] + \theta^\top \left( \Ebb_{ \Pchk_{\theta,\lambda^*_\theta}(\yhat|\xvec,a,y)}
    [\phi(\xvec,\Yhat)]
    - \phi(\xvec,y) \right) \right\}\\
 =& \min_{\theta} \tfrac{1}{n}\!\!\!\!\!\!
    \sum_{(\xvec,a,y) \in \Dcal} \sum_{\widehat{y} \in \mathcal{Y}} - \Qbb(\widehat{y}|\xvec,a,y) \left[ \log \Pbb(\widehat{y}|\xvec,a,y) - \theta^\top \phi(\xvec,\widehat{y}) \right] - \theta^\top\phi(\xvec,y)
\end{align}

Plugging the parametric distribution forms of $\Pbb$ and $\Pchk$, for $\lambda^*_\theta > 0$, we get:
\begin{align}
 & 
\min_{\theta} \tfrac{1}{n}\!\!\!\!\!\!
\sum_{(\xvec,a,y) \in \Dcal}
\begin{cases}
   -\log(\frac{p_{\gamma_1}}{\lambda^*_\theta}) + \theta^\top ( \phi(\xvec,1) - \phi(\xvec,y) ) & \text{if } \gamma_1(a,y), \text{and }  \frac{e^{\theta^\top \phi(\xvec,1)}}{Z_{\theta}(\xvec)} > \frac{p_{\gamma_1}}{\lambda^*_\theta} \\
    -\log(\frac{p_{\gamma_0}}{\lambda^*_\theta}) + \theta^\top ( \phi(\xvec,0)- \phi(\xvec,y))   & \text{if } \gamma_0(a,y),  \text{and } \frac{e^{\theta^\top \phi(\xvec,1)}}{Z_{\theta}(\xvec)} < 1 - \frac{p_{\gamma_0}}{\lambda^*}\\
    \log \sum_{y' \in \mathcal{Y}} e^{\theta^\top \phi(\xvec,y')} - \theta^\top\phi(\xvec,y)& \text{otherwise},
\end{cases}
\end{align}
and for $\lambda^*_\theta < 0$, we get:
\begin{align}
& 
\min_{\theta} \tfrac{1}{n}\!\!\!\!\!\!
\sum_{(\xvec,a,y) \in \Dcal}
\begin{cases}
   -\log(-\frac{p_{\gamma_1}}{\lambda^*_\theta}) + \theta^\top ( \phi(\xvec,0) - \phi(\xvec,y) ) & \text{if } \gamma_1(a,y), \text{and }  \frac{e^{\theta^\top \phi(\xvec,1)}}{Z_{\theta}(\xvec)} < 1 + \frac{p_{\gamma_1}}{\lambda^*_\theta} \\
    -\log(-\frac{p_{\gamma_0}}{\lambda^*_\theta}) + \theta^\top ( \phi(\xvec,1)- \phi(\xvec,y))   & \text{if } \gamma_0(a,y),  \text{and } \frac{e^{\theta^\top \phi(\xvec,1)}}{Z_{\theta}(\xvec)} > - \frac{p_{\gamma_0}}{\lambda^*}\\
    \log \sum_{y' \in \mathcal{Y}} e^{\theta^\top \phi(\xvec,y')} - \theta^\top\phi(\xvec,y)& \text{otherwise},
\end{cases}
\end{align}
and for $\lambda^*_\theta = 0$, we get:
\begin{align}
 & 
\min_{\theta} \tfrac{1}{n}\!\!\!\!\!\!
\sum_{(\xvec,a,y) \in \Dcal} \log\left(\sum_{y' \in \mathcal{Y}} e^{\theta^\top \phi(\xvec_i,y')}\right) - \theta^\top\phi(\xvec,y).
\end{align}
\end{proof}

\subsection{Proof of Theorem \ref{thm:convex}}

\begin{proof}[Proof of Theorem \ref{thm:convex}]

It is easy to see that given $\theta$ and $\lambda_\theta$, the function $\ell_{\theta,\lambda_\theta}(\xvec,a,y)$ is convex  for each sample. It is a linear function with respect to $\theta$ for the truncated cases. For the ``otherwise'' case, we know that $\log Z_\theta(\xvec)$ is convex. Hence, the full objective given $\theta$ and $\lambda_\theta$ is convex (non-negative weighted sum of convex functions is convex). 

Now, the objective function given $\theta$ can be written as:
\begin{align}
    \max_{\lambda} \sum_{(\xvec,a,y) \in \Dcal} \ell_{\theta,\lambda_\theta}(\xvec,a,y).
    \label{eq:obj_max}
\end{align}
Since for each $\lambda$, the function $\sum_{(\xvec,a,y) \in \Dcal} \ell_{\theta,\lambda_\theta}(\xvec,a,y)$ is convex, the objective in Eq. \eqref{eq:obj_max} is also convex with respect to $\theta$ (pointwise supremum of convex functions is convex).
\end{proof}

\subsection{Proof of Theorem \ref{thm:consistency}}

\begin{proof}[Proof of Theorem \ref{thm:consistency}]

A fully expressive feature representation constrains the approximator's distribution in our primal formulation Eq. \eqref{eq:definition} to match the population distribution. Then,  the optimization simplifies to:
\begin{align}
    \Pbb^*(\yhat|\xvec,a,y) = &\argmin_{\Pbb \in  \Delta \cap \Gamma}  \Ebb_{
        P(\xvec,a,y)} 
            \left[-\log \Pbb(\Yhat|\Xbf,A,Y)\right] \\
    =&\argmin_{\Pbb \in  \Delta \cap \Gamma}  -\sum_{(\xvec,a,y)}
        P(\xvec,a,y) \log\left(\Pbb(\Yhat=y|\xvec,a,y)\right)\\
    =&\argmin_{\Pbb \in  \Delta \cap \Gamma}  -\sum_{(\xvec,a,y)}
        P(\xvec,a,y) \log\left(\frac{\Pbb(\Yhat=y|\xvec,a,y)}{P(Y=y|\xvec,a)}\right) - \sum_{(\xvec,a,y)} P(\xvec,a,y) \log \left(P(Y=y|\xvec,a)\right) \\
    =&\argmin_{\Pbb \in  \Delta \cap \Gamma}  -\sum_{(\xvec,a,y)}
        P(\xvec,a,y) \log\left(\frac{\Pbb(\Yhat=y|\xvec,a,y)}{P(Y=y|\xvec,a)}\right) \\
    =&\argmin_{\Pbb \in  \Delta \cap \Gamma} \; D_{\text{KL}} ( P \; \| \; \Pbb ).
\end{align}
This means that the optimal solution of our method when learning from the population distribution with a fully expressive feature representation is the fair predictive distribution that has the minimum KL-divergence from the population distribution.
\end{proof}

\subsection{Proof of Theorem \ref{thm:consistency2}}

\begin{proof}[Proof of Theorem \ref{thm:consistency2}]
For fairness constraints that depend on the true label (e.g., \textsc{E.Opp.} and \textsc{E.Odd.}), as described in \S \ref{sec:inference}, we compute $\Pbb^*(\yhat|\xvec,a)$ using Eq. \eqref{eq:pred-fp} with the input of $\Pbb^*(\yhat|\xvec,a,y)$ and the approximator's distribution to approximate the true distribution. Based on the proof of Theorem \ref{thm:consistency}, we know that, in the limit, the approximator's distribution matches with the true distribution $P(\xvec,a,y)$. Hence, our prediction becomes the standard marginal probability rule (it is no longer an approximation), i.e.:
\begin{align}
     \Pbb^*(\yhat|{\bf \xvec},a) &= \Pbb^*(\yhat|{\bf \xvec},a,y=1) P(y=1|{\bf \xvec},a) + \Pbb^*(\yhat|\xvec,a,y=0) P(y=0|\xvec,a). 
\end{align}
Therefore, our predictor is the marginal predictor distribution computed from the fair predictor's distribution with the closest KL-divergence from the true distribution, marginalized over the true label.
\end{proof}

\section{Optimization Details}

\subsection{Algorithm for finding the optimal fairness parameters}

\label{sec:optimal_lambda}

The inner maximization in Eq. \eqref{eq:dual} 
finds the optimal $\lambda$ that enforces the fairness constraint. From the perspective of the parametric distribution of $\Pbb$, this is equivalent to finding threshold points (e.g., ${p_{\gamma_1}}/{\lambda}$ and $1 - {p_{\gamma_0}}/{\lambda}$) in the $\min$ and $\max$ function of Eq. \eqref{eq:truncate} such that the expectation of the truncated exponential probabilities of $\Pbb$ in group $\gamma_1$ match the one in group $\gamma_0$. Given the value of $\theta$, we find the optimum $\lambda^*$ directly by finding the threshold. 
We first compute the exponential probabilities $P_{e}(\yhat=1|\xvec,a,y) = \exp(\theta^\top \phi(\xvec,1)) / {Z_{\theta}(\xvec)}$ for each examples in $\gamma_1$ and $\gamma_0$. Let $E_1$ and $E_0$ be the sets that contains $P_e$ for group $\gamma_1$ and $\gamma_0$ respectively, and let $\ebar_1$ and $\ebar_0$ be the average of $E_1$ and $E_0$ respectively.

Finding $\lambda^*$ given the sets $E_1$ and $E_0$ requires sorting the probabilities for each set, and then iteratively finding the threshold points for both sets ($t_1$ and $t_0$ respectively) simultaneously as described in Algorithm \ref{alg:lambda}. Without loss of generality\footnote{For the case when $\ebar_1 < \ebar_0$, we flip the group membership and then $\lambda^*$ is the  negative of the solution produced by the algorithm. When $\ebar_1 = \ebar_0$, the exponential probabilities are already fair, and we set $\lambda^* = 0$.}\!, the algorithm assumes that $\ebar_1 > \ebar_0$. 
It find the threshold points by traversing the sorted list of points in $E_1$ and $E_0$ until it find the thresholds that ensure the equality of the average truncated probabilities in both groups.

        \begin{algorithm}[H]
        	\caption{Find $\lambda^*$ given $E_1$ and $E_0$}
        	\label{alg:lambda}
        	\begin{algorithmic}[1]
        		\STATE {\bfseries Input:} $(E_1, E_0)$, s.t. $\ebar_1 > \ebar_0$
        		\STATE Sort $E_1$ in decreasing order
                \STATE Sort $E_0$ in increasing order
                \STATE Calculate the difference $\dbar = \ebar_1 - \ebar_0$
                \STATE $t_1 \leftarrow 1, \; t_0 \leftarrow 0$ \COMMENT{thresholds for $E_1$ and $E_0$ respectively}
                \STATE Set gain to be 0.
        		\WHILE {the gain is less than $\dbar$}
        		\STATE Calculate two candidates for the next move:
        		       \begin{enumerate}[noitemsep,topsep=0pt,label=(\arabic*),itemindent=0pt]
        		       \item move $t_1$ to the next $P_e$ in $E_1$ list
        		       \item move $t_0$ to the next $P_e$ in $E_0$ list
        		       \end{enumerate}
        		\STATE Calculate the gain for each move and the effect of the move for the other group.
        		\STATE Choose the move that has minimum gain
        		\ENDWHILE 
        		\STATE Calculate threshold that produces gain equal to $\dbar$, which is located between the last move in the loop and the threshold before the move
        		\STATE Calculate $\lambda^*$ based on the threshold
        		\STATE {\bfseries return:} $\lambda^*$
        	\end{algorithmic}
        \end{algorithm}

The runtime of 
Algorithm \ref{alg:lambda} is dominated by sorting, i.e., $O(n \log n)$ time. 
However, if we perform subgradient based optimization on $\theta$, the value of the current $\theta$ in each iteration does not change much, and neither do the exponential probabilities in each group. 
By maintaining the sorted list in each iteration as the basis index, the next iteration will have the probabilities in a nearly sorted order. Therefore, we can improve the sorting cost requirement by running sorting algorithms that work best on a nearly sorted list (e.g. insertion sort, Timsort, or $P^3$-sort) with run times approaching $O(n)$ as the list is close to being fully sorted \cite{chandramouli2014patience}.

\subsection{Subgradient-based Optimization}

\label{sec:subgrad-optim}

Our optimization objective is:
\begin{align}
 \Lcal &=\min_{\theta}  \tfrac{1}{n}
    \textstyle\sum_{(\xvec,a,y) \in \Dcal} \ell_{\theta,\lambda^*_\theta}(\xvec,a,y), \quad \text{ where: }  \\
    &\ell_{\theta,\lambda^*}(\xvec,a,y) = 
    - \theta^\top \phi(\xvec,y) + 
    \begin{cases}
        -\log(\frac{p_{\gamma_1}}{\lambda^*_\theta}) + \theta^\top ( \phi(\xvec,1) 
        ) & \text{if } \gamma_1(a,y) \wedge T(\xvec,\theta) \wedge \lambda^*_\theta > 0 \\
        -\log(\frac{p_{\gamma_0}}{\lambda^*_\theta}) + \theta^\top ( \phi(\xvec,0) 
        ) & \text{if } \gamma_0(a,y) \wedge T(\xvec,\theta) \wedge \lambda^*_\theta > 0 \\
        -\log(-\frac{p_{\gamma_1}}{\lambda^*_\theta}) + \theta^\top ( \phi(\xvec,0) 
        ) & \text{if } \gamma_1(a,y) \wedge T(\xvec,\theta) \wedge \lambda^*_\theta < 0\\
        -\log(-\frac{p_{\gamma_0}}{\lambda^*_\theta}) + \theta^\top ( \phi(\xvec,1) 
        ) & \text{if } \gamma_0(a,y) \wedge T(\xvec,\theta) \wedge \lambda^*_\theta < 0 \\
        \log(Z_\theta(\xvec))  & \text{otherwise}. \notag
    \end{cases}
\end{align}

Taking the gradient of the objective with respect to $\theta$, for $\lambda^*_\theta > 0$, we get:
\begin{align}
 \partial_\theta \Lcal \ni& 
\min_{\theta} \tfrac{1}{n}\!\!\!\!\!\!
\sum_{(\xvec,a,y) \in \Dcal}
\begin{cases}
     \phi(\xvec,1) - \phi(\xvec,y)  & \text{if } \gamma_1(a,y), \text{and }  \frac{e^{\theta^\top \phi(\xvec_i,1)}}{Z_{\theta}(\xvec_i)} > \frac{p_{\gamma_1}}{\lambda^*_\theta} \\
     \phi(\xvec,0)- \phi(\xvec,y)   & \text{if } \gamma_0(a,y),  \text{and } \frac{e^{\theta^\top \phi(\xvec_i,1)}}{Z_{\theta}(\xvec_i)} < 1 - \frac{p_{\gamma_0}}{\lambda^*}\\
    \textstyle\sum_{y' \in \mathcal{Y}}\tfrac{ e^{\theta^\top\phi(\xvec,y')}}{Z_\theta(\xvec)}\phi(\xvec,y') - \phi(\xvec,y) & \text{otherwise},
\end{cases}
\end{align}
and for $\lambda^*_\theta < 0$, we get:
\begin{align}
 \partial_\theta \Lcal \ni& 
\min_{\theta} \tfrac{1}{n}\!\!\!\!\!\!
\sum_{(\xvec,a,y) \in \Dcal}
\begin{cases}
   \phi(\xvec,0) - \phi(\xvec,y)  & \text{if } \gamma_1(a,y), \text{and }  \frac{e^{\theta^\top \phi(\xvec_i,1)}}{Z_{\theta}(\xvec_i)} < 1 + \frac{p_{\gamma_1}}{\lambda^*_\theta} \\
     \phi(\xvec,1)- \phi(\xvec,y)   & \text{if } \gamma_0(a,y),  \text{and } \frac{e^{\theta^\top \phi(\xvec_i,1)}}{Z_{\theta}(\xvec_i)} > - \frac{p_{\gamma_0}}{\lambda^*}\\
    \textstyle\sum_{y' \in \mathcal{Y}}\tfrac{ e^{\theta^\top\phi(\xvec,y')}}{Z_\theta(\xvec)}\phi(\xvec,y') - \phi(\xvec,y) & \text{otherwise},
\end{cases}
\end{align}
and for $\lambda^*_\theta = 0$, we get:
\begin{align}
 \partial_\theta \Lcal \ni& 
\min_{\theta} \tfrac{1}{n}\!\!\!\!\!\!
\sum_{(\xvec,a,y) \in \Dcal} \sum_{y' \in \mathcal{Y}}\tfrac{ e^{\theta^\top\phi(\xvec,y')}}{Z_\theta(\xvec)}\phi(\xvec,y') - \phi(\xvec,y).
\end{align}

This can be simplified as:
\begin{align}
 \partial_\theta \Lcal \ni& 
\min_{\theta} \tfrac{1}{n}\!\!\!\!\!\!
\sum_{(\xvec,a,y) \in \Dcal}
\begin{cases}
     \phi(\xvec,1) - \phi(\xvec,y)  & \text{if } \gamma_1(a,y), \text{and }  \frac{e^{\theta^\top \phi(\xvec_i,1)}}{Z_{\theta}(\xvec_i)} > \frac{p_{\gamma_1}}{\lambda^*_\theta} , \text{and } \lambda^*_\theta > 0, \\
     & \text{or } \text{if } \gamma_0(a,y),  \text{and } \frac{e^{\theta^\top \phi(\xvec_i,1)}}{Z_{\theta}(\xvec_i)} > - \frac{p_{\gamma_0}}{\lambda^*} , \text{and }  \lambda^*_\theta < 0   \\
     \phi(\xvec,0)- \phi(\xvec,y)   & \text{if } \gamma_0(a,y),  \text{and } \frac{e^{\theta^\top \phi(\xvec_i,1)}}{Z_{\theta}(\xvec_i)} < 1 - \frac{p_{\gamma_0}}{\lambda^*} , \text{and } \lambda^*_\theta > 0,\\
     & \text{or } \text{if } \gamma_1(a,y), \text{and }  \frac{e^{\theta^\top \phi(\xvec_i,1)}}{Z_{\theta}(\xvec_i)} < 1 + \frac{p_{\gamma_1}}{\lambda^*_\theta} , \text{and }  \lambda^*_\theta < 0   \\
    \textstyle\sum_{y' \in \mathcal{Y}}\tfrac{ e^{\theta^\top\phi(\xvec,y')}}{Z_\theta(\xvec)}\phi(\xvec,y') - \phi(\xvec,y) & \text{otherwise},
\end{cases}
\end{align}
or using our $T(\xvec,\theta)$ notation as:
\begin{align}
& \qquad \qquad \tfrac{1}{n}
    \sum_{(\xvec,a,y) \in \Dcal} g_{\theta,\lambda^*_\theta}(\xvec,a,y) \in  \partial_\theta \Lcal, \text{ where: } 
    \label{eq:gradient}
    \\
g_{\theta,\lambda^*_\theta}(\xvec,a,y) =&\begin{cases}
    \phi(\xvec,1) - \phi(\xvec,y), \qquad \text{if } (\gamma_1(a,y) \wedge T(\xvec,\theta) \wedge \Ibb[\lambda_\theta^* \!>\! 0]) \vee (\gamma_0(a,y) \wedge T(\xvec,\theta) \wedge \Ibb[\lambda^*_\theta \!<\! 0]) \\
    \phi(\xvec,0) - \phi(\xvec,y), \qquad \text{if } (\gamma_0(a,y) \wedge T(\xvec,\theta) \wedge \Ibb[\lambda_\theta^* \!>\! 0]) \vee (\gamma_1(a,y) \wedge T(\xvec,\theta) \wedge \Ibb[\lambda^*_\theta \!<\! 0]) \\
    \sum_{y' \in \mathcal{Y}}\tfrac{ \exp(\theta^\top\!\phi(\xvec,y'))}{Z_\theta(\xvec)}\phi(\xvec,y') \!-\! \phi(\xvec,y), \qquad \text{otherwise. }
\end{cases} \notag 
\end{align}

\end{document}